\def\ARXIV{1}
\documentclass[final,12pt]{custom_colt} 

\usepackage{mathtools}
\usepackage{amsfonts}
\usepackage{bbm}

\PassOptionsToPackage{hyphens}{url}\usepackage{hyperref}

\hypersetup{
    colorlinks,
    linkcolor={red!50!black},
    citecolor={blue!50!black},
    urlcolor={blue!80!black}
}

\usepackage{tikz}
\usetikzlibrary{shapes,matrix,arrows,calc,positioning,fit,bayesnet,angles}

\usepackage{customCommands}

\usepackage{appendix}

\usepackage{comment}

\newcounter{constnum}
\newcommand{\const}[1]{\refstepcounter{constnum}\label{#1}}

\title[PSGLA for Sampling and Learning]{Projected Stochastic Gradient Langevin Algorithms for Constrained Sampling and Non-Convex Learning}
\usepackage{times}

\coltauthor{%
 \Name{Andrew Lamperski} \Email{alampers@umn.edu}\\
 \addr 200 Union St. Se, Keller Hall 4-174, Minneapolis, MN 55455, USA
}

\begin{document}

\maketitle

\begin{abstract}%
  Langevin algorithms are gradient descent methods with additive noise. They have been used for decades in Markov Chain Monte Carlo (MCMC) sampling, optimization, and learning. 
Their convergence properties for unconstrained non-convex optimization and learning problems have been studied widely in the last few years.
Other work has examined projected Langevin algorithms for sampling from log-concave distributions restricted to convex compact sets. 
For learning and optimization, log-concave distributions correspond to convex losses.
In this paper, we analyze the case of non-convex losses with compact convex constraint sets and IID external data variables. We term the resulting method the projected stochastic gradient Langevin algorithm (PSGLA).
We show the algorithm achieves a deviation of $O(T^{-1/4}(\log T)^{1/2})$ from its target distribution in $1$-Wasserstein distance.
For optimization and learning, we show that the algorithm achieves $\epsilon$-suboptimal solutions, on average, provided that it is run for a time that is polynomial in $\epsilon^{-1}$ and slightly super-exponential in the problem dimension. 

\end{abstract}


\section{Introduction}
Langevin dynamics originate in the study of statistical physics \cite{coffey2012langevin}, and have a long history of applications to Markov Chain Monte Carlo (MCMC) sampling \cite{roberts1996exponential}, non-convex optimization  \cite{gelfand1991recursive,borkar1999strong}, and machine learning \cite{welling2011bayesian}. Langevin algorithms amount to gradient descent augmented with additive Gaussian noise. 
This additive noise enables the algorithms to escape saddles and local minima.
For optimization and learning, this enables the algorithms to find near optimal solutions even when the losses are non-convex. For sampling, Langevin algorithms give a simple approach to produce samples that converge to target distributions which are not log-concave.

\paragraph{Related Work.}
A large amount of progress on the non-asymptotic analysis of Langevin algorithms has been reported in recent years. This work has two main streams: 1) unconstrained non-convex problems and 2) constrained convex problems. These works will be reviewed below.

The bulk of the  recent work on non-asymptotic analysis of Langevin algorithms has examined  unconstrained problems 
\cite{raginsky2017non,majka2020nonasymptotic,fehrman2020convergence,chen2020stationary,erdogdu2018global,durmus2017nonasymptotic,chau2019stochastic,xu2018global,cheng2018sharp,ma2019sampling}. The basic algorithm in the unconstrained case has the form:
$$
\bx_{k+1} = \bx_k-\eta \nabla_x f(\bx_k,\bz_k) + \sqrt{\frac{2\eta}{\beta}} \bw_k,
$$
where $\bx_k$ is the decision variable, $\bz_k$ are external random variables, $\bw_k$ is Gaussian noise, and $\eta$ and $\beta$  are parameters. In a learning context, $\bz_k$  correspond to data, $\bx_k$ are parameters of a  model, and $f(x,z)$ is a loss function that describes how well the model parameters fit the data. With no Gaussian noise, $\bw_k$, this algorithm reduces to stochastic gradient descent. 

A breakthrough was achieved in \cite{raginsky2017non}, which gave non-asymptotic bounds in the case that $f(x,z)$ is non-convex in $x$ and $\bz_k$ are independent identically distributed (IID). A wide number of improvements and variations on these results have since been obtained in works such as
\cite{majka2020nonasymptotic,fehrman2020convergence,chen2020stationary,erdogdu2018global,durmus2017nonasymptotic,chau2019stochastic,xu2018global,cheng2018sharp,ma2019sampling}. In particular, \cite{chau2019stochastic} achieves tighter performance guarantees and extends to the case that $\bz_k$ is a mixing process. 

For problems with constraints, most existing work focuses convex losses over compact convex constraint sets with no external variables $\bz_k$. Most closely related to our work is that of \cite{bubeck2015finite,bubeck2018sampling} which augments the Langevin algorithm with a projection onto the constraint set.  Proximal-type algorithms were examined \cite{brosse2017sampling}. Variations on mirror descent were examined in \cite{ahn2020efficient,hsieh2018mirrored,zhang2020wasserstein,krichene2017acceleration}. 

Recent work of \cite{wang2020fast} examines Langevin dynamics on Riemannian manifolds. In this case, the losses may be non-convex, but still there are no external variables, $\bz_k$.  It utilizes results from diffusion theory to give convergence with respect to Kullback-Liebler (KL) divergence. Many of the ideas in that paper could likely be translated to the current setting. However, such KL divergence bounds become degenerate if the algorithm is initialized as a constant value, e.g. $\bx_0=0$. 
In contrast, our work focuses on bounds in the $1$-Wasserstien distance, which gives well-defined bounds as long as the initialization is feasible for the constraints.


\paragraph{Contributions.} This paper gives non-asymptotic convergence bounds for Langevin algorithms for problems that are constrained to a compact convex set. In particular, we examine a generalized version of the algorithm examined in \cite{bubeck2018sampling,bubeck2015finite}. As discussed above, the existing works on constrained Langevin methods (aside from the Riemannian manifold results of \cite{wang2020fast}) focus on convex loss functions, and none consider external random variables. This paper examines the case of non-convex losses with IID external randomness. For the purpose of sampling, it is shown that after $T$ steps, the error from the target in the $1$-Wasserstein is of  $O(T^{-1/4}(\log T)^{1/2})$. For optimization and learning, this bound is used to show that the algorithm can achieve a suboptimality of $\epsilon$ in a number of steps that is polynomial in $\epsilon$ and slightly superexponential in the dimension of $\bx_k$. To derive the bounds, a novel result on contractions for reflected stochastic differential equations is derived. 

\section{Setup}

\subsection{Notation and Terminology.}
$\bbR$ denotes the set of real numbers while $\bbN$ denotes the set of non-negative integers. The Euclidean norm over $\bbR^n$ is denoted by $\|\cdot \|$. 

Random variables will be denoted in bold. If $\bx$ is a random variable, then $\bbE[\bx]$ denotes its expected value and $\cL(\bx)$ denotes its law. IID stands for independent, identically distributed.
The indicator function is denoted by $\indic$. If $P$ and $Q$ are two probability measures over $\bbR^n$, then the $1$-Wasserstein distance between them with respect the Euclidean norm   is denoted by $W_1(P,Q)$. 

Throughout the paper, $\cK$ will denote a compact convex subset of $\bbR^n$ of diameter  $D$ such that a ball of radius $r>0$ around the origin is contained in $\cK$. The boundary of $\cK$ is denoted by $\partial \cK$. The normal cone of $\cK$ at a point $x$ is denoted by $N_{\cK}(x)$. The convex projection onto $\cK$ is denoted by $\Pi_{\cK}$. 
\subsection{The Project Stochastic Gradient Langevin Algorithm}

For integers $k$ let  $\hat \bw_k \sim\cN(0,I)$ be IID Gaussian random variables and let $\bz_k$ be IID random variables whose properties will be described later. Assume that $\bz_i$ and $\hat \bw_j$ are independent for all $i,j\in\bbN$.

Assume that the initial value of $\bx_0\in\cK$ is independent of $\bz_i$ and $\hat \bw_j$. Then
the projected stochastic gradient Langevin algorithm has the form:
\begin{equation}
  \label{eq:projectedLangevin}
\bx_{k+1} = \Pi_{\cK}\left(\bx_k -\eta \nabla_x f(\bx_k,\bz_k) +
  \sqrt{\frac{2\eta}{\beta}} \hat\bw_{k}\right),
\end{equation}
with $k$ an integer. Here $\eta>0$ is the step size parameter and $\beta >0$ is a noise parameter.

Let $\bar f(x) = \bbE[f(x,\bz)]$, where the expectation is over $\bz$, which has the same distribution as $\bz_k$. We will assume that $\nabla_x f(x,\bz)-\nabla_x \bar f(x)$ are uniformly sub-Gaussian for each $x\in\mathbb{R}^n$. That is, there is a number $\sigma >0$ such that for all $\alpha\in \mathbb{R}^n$, the following bound holds:
\begin{equation}
  \label{eq:subgaussian}
  \bbE\left[
    \exp\left(
      \alpha^\top \left (
        \nabla_x f(x,\bz)-\nabla_x \bar f(x)
        \right)
      \right)
  \right] \le e^{\sigma^2 \|\alpha\|^2/2}.
\end{equation}
The uniform sub-Gaussian property holds under the following conditions:
\begin{itemize}
\item {\bf Gradient Noise:} $\nabla_x f(x,\bz) = \nabla_x \bar f(x) + \bz$ with $\bz$ sub-Gaussian.
\item {\bf Lipschitz Gradients and Strongly Log-Concave $\bz$:} $\nabla_x f(x,z)$ is Lipschitz in $z$ and $\bz$ has a density of the form $e^{-U(z)}$ with $\nabla^2 U(z) \succeq \kappa I$ for all $z$. Here $\kappa >0$ and the inequality is with respect to the positive semidefinite partial order. (See Theorem 5.2.15  of \cite{vershynin2018high}.) 
\item {\bf Convex Gradients and Bounded $\bz$:} Each component $\frac{\partial f(x,z)}{\partial x_i}$ is convex in $z$ and $\bz$ is bounded with independent components. (See Theorem 3.24 of \cite{wainwright2019high}.)
  \end{itemize}

  For learning, the last two conditions are the most useful, since they give general classes of losses and variables for which the method can be applied. 
  In particular, the second case applies to many common scenarios. It includes Gaussian $\bz$ as a special case, and it can be applied to neural networks with smooth activation functions. 
  A variety of more specialized cases in which the sub-Gaussian condition holds are presented in Chapter 5 of \cite{vershynin2018high}.
  Future work will relax the uniform sub-Gaussian assumption and the requirement of IID $\bz_k$.

We assume  that for each $z$, $\nabla_x f(x,z)$ is $\ell$-Lipschitz in $x$, i.e. $\|\nabla_x f(x_1,z)-\nabla_x f(x_2,z)\| \le \ell \|x_1-x_2\|$. The mean function, $\bar f$, is assumed to be $u$-smooth, so that $\|\nabla_x \bar f(x)\|\le u$ for all $x\in\cK$. 
The assumptions on $\bar f$ imply that we can have $u\le \|\nabla_x \bar f(0)\| + \ell D$ and that $\bar f$ is $u$-Lipschitz.

In \cite{bubeck2018sampling}, the case with $\bar f$ is convex and no $\bz_k$ variables is
studied. 
It is shown that by choosing the step size, $\eta$, appropriately, the
law of $\bx_k$  is given approximately given by $\pi_{\beta \bar f}$, which is
defined by
\begin{equation}
  \label{eq:gibbs}
  \pi_{\beta \bar f}(A) = \frac{\int_A e^{-\beta \bar f(x)} dx}{\int_K e^{-\beta \bar f(x)} dx}.
\end{equation}

In this paper, we will bound the convergence of
\eqref{eq:projectedLangevin} to \eqref{eq:gibbs} in the case of
non-convex $f$ with external random variables $\bz_k$.  

\section{Main Results}

\subsection{Convergence of the Law of the Iterates}

The following is the main result of the paper. It is proved in Subsection~\ref{ss:mainProof}. 

\begin{theorem}
  \label{thm:nonconvexLangevin}
  \const{globalContract}
  \const{globalConst}
  Assume that $\eta \le 1/2$. 
  There  are positive constants $a$, $c_{\ref{globalContract}}$ and $c_{\ref{globalConst}}$ such that for all integers $k\ge 4$, the following bound holds:
  \begin{equation*}
    W_1(\cL(\bx_k),\pi_{\beta \bar f}) \le c_{\ref{globalContract}}e^{-\eta a k} + c_{\ref{globalConst}} (\eta\log k)^{1/4}
  \end{equation*}
  In particular, if $\eta = \frac{\log T}{4aT}$ and $T\ge 4$, then 
  \begin{equation*}
    \label{eq:optEta}
  W_1(\cL(\bx_T),\pi_{\beta \bar f}) \le \left(c_{\ref{globalContract}}+\frac{c_{\ref{globalConst}}}{(4a)^{1/4}}\right) T^{-1/4}(\log T)^{1/2}.
  \end{equation*}
\end{theorem}

The constants depend on  the dimension of $\bx_k$, $n$, the noise parameter, $\beta$, the Lipschitz constant, $\ell$, the diameter, $D$, the size of the inscribed ball $r$, and the smoothness constant, $u$. The specific form of the constants will be derived in the proof. For applications, it is useful to know how the constants depend on the dimension, $n$, and the noise parameter, $\beta$. 
The result below indicates that the algorithm exhibits two distinct regimes in which convergence is fast and slow, respectively. It is proved in Appendix~\ref{sec:constants}.

\begin{proposition}
  \label{prop:constants}
  \const{smallA}
  \const{largeBeta}
The constants $c_{\ref{globalContract}}$ and $c_{\ref{globalConst}}$ grow linearly with $n$. 
If $D^2\ell \beta <8$, then we can set $a=\frac{4}{D^2\beta}\ge \frac{\ell}{2}$, while $c_{\ref{globalContract}}$ and $c_{\ref{globalConst}}$ grow polynomially with respect to $\left(1-\frac{D^2\ell \beta}{8}\right)^{-2}$ and $\beta^{-1/4}$. 
In general, we have a positive constant $c_{\ref{smallA}}$ and a monotonically increasing polynomial $p$ (independent of $\eta$ and $\beta$) such that for all $\beta>0$, the following bounds hold:
\begin{align*}
  a&\ge c_{\ref{smallA}} \beta \exp\left(-\frac{D^2\ell \beta}{4}\right) \\
  \max\left\{c_{\ref{globalContract}},c_{\ref{globalConst}}\right\} &\le p(\beta^{-1/4})
    \exp\left(\frac{3 D^2\ell \beta}{4}\right).
\end{align*}
\end{proposition}

\subsection{Application to Optimization and Learning}

The following result shows that the $\bx_k$ can be made arbitrarily near optimal, but the required time may be slightly super-exponential with respect to problem dimension, $n$. 
It is proved in Appendix~\ref{sec:nearOpt}.
\begin{proposition}
  \const{subOpt}
  \label{prop:suboptimality}
  Assume that $\eta\le 1/2$.
  There is a positive constant, $c_{\ref{subOpt}}$ such that for all $k\ge 4$, the following bound holds:
  \begin{equation*}
    \bbE[\bar f(\bx_k)] \le \min_{x\in \cK} \bar f(x) + uW_1(\cL(\bx_k),\pi_{\beta \bar f}) + \frac{n\log(c_{\ref{subOpt}}\max\{1,\beta\})}{\beta}.
  \end{equation*}
  In particular, given any $\rho > 4$ and any $\zeta > 1$, there are choices of $\eta$, $\beta$, and $T$, along with positive numbers $m(\rho,\zeta)$ and $\alpha(\rho,\zeta)$, such that for any suboptimality level, $\epsilon>0$, the following implication holds:
  \begin{equation*}
    T\ge \frac{m(\rho,\zeta)}{\epsilon^\rho} \exp\left(\alpha(\rho,\zeta) n^\zeta\right) \implies \bbE[\bar f(\bx_T)]\le \min_{x\in \cK} \bar f(x) + \epsilon.
  \end{equation*}
\end{proposition}

\subsection{The Auxiliary Processes Used for the Main Bound}
\label{ss:processes}

Similar to other analyses of Langevin methods, e.g. \cite{raginsky2017non,bubeck2018sampling,chau2019stochastic}, the proof of Theorem~\ref{thm:nonconvexLangevin} utilizes a collection of auxiliary stochastic processes that fit between the algorithm iterates from (\ref{eq:projectedLangevin}) and a stationary Markov process with state distribution given by (\ref{eq:gibbs}). 

We will embed the iterates of the algorithm into continuous time by setting $\bx_t^A=\bx_{\floor*{t}}$. The $A$ superscript is used to highlight the connection between this process and the algorithm. The Gaussian variables $\hat \bw_k$ can be realized as $\hat \bw_k = \bw_{k+1}-\bw_k$ where $\bw_t$ is a Brownian motion.

We will let $\bx_t^C$ be a continuous approximation of $\bx_t^A$ and we will let $\bx_t^M$ be a variation on the process $\bx_t^C$ in which averages out the effect of the $\bz_k$ variables. The proof will proceed by showing that the law of $\bx_t^M$ converges exponentially to \eqref{eq:gibbs}, that $\bx_t^C$ has a similar law to $\bx_t^M$, and that $\bx_t^A$ has a similar law to $\bx_t^C$. Below we make these statements more precise.  

The continuous approximation of the algorithm is defined by the following reflected stochastic differential equation (RSDE):
\begin{equation}
  \label{eq:continuousProjectedLangevin}
  d\bx^C_t = -\eta \nabla_x f(\bx^C_t,\bz_{\floor*{t}}) dt +
  \sqrt{\frac{2\eta}{\beta}} d\bw_t - \bv_t^C d\bmu^C(t).
\end{equation}

Here $-\int_0^t \bv_s^Cd\bmu^C(s)$ is a bounded variation reflection process that ensures that $\bx_t^C\in\cK$ for all $t\ge 0$, as long as $\bx_0^C\in\cK$.  In particular, the measure $\bmu^C$ is such that $\bmu^C([0,t])$ is finite,  $\bmu^C$ supported on $\{s|\bx_s^C\in\partial \cK\}$, and $\bv_s^C\in N_{\cK}(\bx_s^C)$ where $N_{\cK}(x)$ is the normal cone of $\cK$ at $x$. Under these conditions, the reflection process is uniquely defined and $\bx^C$ is the unique solution to the Skorokhod problem for the process defined by:
\begin{equation*}
  \by^C_t = \bx_0^C + \sqrt{\frac{2\eta}{\beta}} \bw_t - \eta \int_0^t \nabla_x
  f(\bx^C_s,\bz_{\floor*{s}}) ds
\end{equation*}
See Appendix~\ref{appsec:skorohod} for more details on the Skorokhod
problem.

For compact notation, we denote the Skorkohod solution for given trajectory, $\by$, by 
$\cS(\by)$. So, the fact that $\bx^C$ is the solution to the Skorokhod
problem for $\by^C$ will be denoted succinctly by $\bx^C=\cS(\by^C)$.

The averaged version of $\bx_t^C$, denoted by $\bx_t^M$, where the $M$ corresponds to ``mean'', is 
defined by:
\begin{equation}
  \label{eq:averagedLangevin}
  d\bx^M_t = -\eta \nabla_x \bar{f}(\bx^M_t) dt +
  \sqrt{\frac{2\eta}{\beta}} d\bw_t - \bv_t^M d\bmu^M(t).
\end{equation}
Again $-\int_0^t \bv_s^Md\bmu^M(s)$ is the unique reflection process that ensures that $\bx_t^M\in\cK$ for all $t$ whenever $\bx_0^M\in\cK$. 
By construction, $\bx_t^M$
satisfies the Skorokhod
problem for the continuous process defined by
\begin{equation*}
  \by^M_t = \bx_0^M + \sqrt{\frac{2\eta}{\beta}} \bw_t - \eta \int_0^t \nabla_x
  \bar{f}(\bx^M_s) ds.
\end{equation*}
See Appendix~\ref{appsec:skorohod} for more details on the Skorokhod
problem.

The following lemmas describe the relationships between the laws all of these processes. They are proved in Sections~\ref{sec:langevinDiffusionBound}, \ref{sec:AtoC}, \ref{sec:averaging} respectively. 
\begin{lemma}
  \label{lem:convergeToStationary}
  There are positive constants $c_{\ref{globalContract}}$ and $a$ such that for all $t\ge 0$
  \begin{equation*}
     W_1(\cL(\bx_t^{M}),\pi_{\beta \bar f})\le c_{\ref{globalContract}}e^{-\eta at}.
   \end{equation*}
 \end{lemma}
 
 \begin{lemma}
   \label{lem:AtoC}
   \const{AtoC}
   Assume that $\bx_0^A=\bx_0^C\in\cK$ and $\eta \le 1/2$. There is a positive constant, 
$c_{\ref{AtoC}}$, such that for all $t\ge 4$,
  $$
  W_1(\cL(\bx_t^A),\cL(\bx_t^C))\le c_{\ref{AtoC}} \left( \eta \log t\right)^{1/4}.
  $$
\end{lemma}

\begin{lemma}
  \label{lem:CtoM}
  \const{CtoM}
  Assume that $\bx_0^M=\bx_0^C\in\cK$ and $\eta\le 1/2$. There is a positive constant, $c_{\ref{CtoM}}$ such that for all $t\ge 0$,
  $$
  W_1(\cL(\bx_t^M),\cL(\bx_t^C))\le c_{\ref{CtoM}} \eta^{1/4}.
  $$
\end{lemma}


Most of the rest of the paper focuses on proving these lemmas. Assuming that these lemmas hold, the main result now has a short proof, which we describe next.  

\subsection{Proof of Theorem~\ref{thm:nonconvexLangevin}}
\label{ss:mainProof}
Recall that $\bx_k^A=\bx_k$ for all integers $k\in\bbN$. Assume that $\bx_0=\bx_0^A=\bx_0^C=\bx_0^M$. 
The triangle inequality followed by Lemmas~\ref{lem:convergeToStationary}, \ref{lem:AtoC}, and \ref{lem:CtoM} shows that
\begin{align*}
  W_1(\cL(\bx_k),\pi_{\beta f}) &\le W_1(\cL(\bx_k^A),\cL(\bx_k^C))+W_1(\cL(\bx_k^C),\cL(\bx_k^M))+W_1(\cL(\bx_k^M),\pi_{\beta f}) \\
  & \le c_{\ref{globalContract}}e^{-\eta ak} + c_{\ref{AtoC}} \left( \eta \log k\right)^{1/4} +  c_{\ref{CtoM}} \eta^{1/4}.
\end{align*}
The result now follows by noting that $\log k \ge 1$ for $k\ge 4$ setting $c_{\ref{globalConst}}=c_{\ref{AtoC}}+c_{\ref{CtoM}}$. The specific bound when $\eta = \frac{\log T}{4aT}$ arises from direct computation. 
\hfill$\blacksquare$

\section{Contractions for the Reflected SDEs}

\label{sec:langevinDiffusionBound}

In this section, we will show how the laws of the processes $\bx_t^C$  and $\bx_t^M$ are contractive with respect to a specially constructed Wasserstein distance. By relating this specially constructed distance with $W_1$ we will prove Lemma~\ref{lem:convergeToStationary} which states that $\cL(\bx_t^M)$ converges to $\pi_{\beta f}$ exponentially with respect to $W_1$.  
The contraction will be derived by an  extension of the
reflection coupling argument of \cite{eberle2016reflection} to the
case of reflected SDEs with external randomness (from $\bz_k$). This result may be of independent
interest.

\begin{proposition}
  \label{prop:W1contract}
  \const{W1mult}
  There are positive
  constants $a$ and $c_{\ref{W1mult}}$ such that  for any two solutions,
  $\bx_t^{C,1}$ and $\bx_t^{C,2}$ to the
  continuous-time RSDE, \eqref{eq:continuousProjectedLangevin},
  their laws converge according to
     \begin{equation}
    \label{eq:W1contract}
    W_1(\cL(\bx_t^{C,1}),\cL(\bx_t^{C,2}))\le c_{\ref{W1mult}} e^{-\eta at} W_1(\cL(\bx_0^{C,1}),\cL(\bx_0^{C,2})) 
  \end{equation}

   To define the constants, let the natural frequency and damping
   ratio be given by
  \begin{equation}
    \label{eq:oscillatorConstants}
    \omega_N = \frac{\sqrt{a\beta}}{2} \quad \textrm{and} \quad
    \xi = \frac{D\ell}{4}\sqrt{\frac{\beta}{a}}.
  \end{equation}

  The constants can always be set to
  \begin{align*}
    a&=\frac{D^2\ell^2\beta}{16}\left(1-\tanh^2\left(
       \frac{D^2\ell\beta}{8}\right)\right) \\
    c_{\ref{W1mult}}&= \frac{e^{D\omega_N \xi}}{\cosh(D\omega_N\sqrt{\xi^2-1})-\frac{\xi}{\sqrt{\xi^2-1}}\sinh(D\omega_N\sqrt{1-\xi^2})}
  \end{align*}
  
   When $D^2\ell \beta <8$, a larger decay constant, $a$, can be
   defined by setting 
   \begin{align*}
     a &= \frac{4}{D^2\beta} \\
     c_{\ref{W1mult}} &= \frac{e^{D\omega_N \xi}}{\cos(D\omega_N\sqrt{1-\xi^2})-\frac{\xi}{\sqrt{1-\xi^2}}\sin(D\omega_N\sqrt{1-\xi^2})}.
   \end{align*}
\end{proposition}

\begin{proof}
  We will follow the main idea behind
  \cite{eberle2016reflection}.  We will correlate the
  solutions using reflection coupling, and then construct a distance
  function, $h$, from the coupling. Then $h$ will be used to construct a
  Wasserstein distance for which the laws $\cL(\bx_t^{C,1})$ and
  $\cL(\bx_t^{C,2})$ converge exponentially.
  The desired bound is found by
  comparing this auxiliary distance to the classical $W_1$ distance. 

  Let $\brho_t = \bx_t^{C,1}-\bx_t^{C,2}$, $\bu_t = \brho_t /
  \|\brho_t\|$ and $\btau = \inf\{t | \bx_t^{C,1}=\bx_t^{C,2}\}$. Note
  that $\btau$.  The \emph{reflection coupling} between $\bx_t^{C,1}$ and
  $\bx_t^{C,2}$ is defined by:  
  \begin{subequations}
    \label{eq:reflectionCoupling}
  \begin{align}
    d\bx_t^{C,1} &= -\eta \nabla_x
                   f(\bx_t^{C,1},\bz_{\floor*{t}})+\sqrt{\frac{2\eta}{\beta}}d\bw_t
                   -\bv_t^{C,1}d\bmu^{C,1}(t)\\
    d\bx_t^{C,2} &= -\eta \nabla_x
                   f(\bx_t^{C,2},\bz_{\floor*{t}})+\sqrt{\frac{2\eta}{\beta}}(I-2\bu_t
                   \bu_t^\top\indic(t < \btau))d\bw_t - \bv_t^{C,2}d\bmu^{C,2}(t).
  \end{align}
\end{subequations}
Here $I$ is the $n\times n$ identity matrix. 
  Also, $\bvarphi_t^1=-\int_0^t\bv_s^{C,1}d\bmu^{C,1}(s)$ and $\bvarphi_t^2=-\int_0^t
  \bv_s^{C,2} d\bmu^{C,2}(s)$ are the unique projection processes that
  ensure that respective
  Skorkhod problem solutions, $\bx_t^{C,1}$ and $\bx_t^{C,2}$,
  remain in $\cK$.

  The processes from \eqref{eq:reflectionCoupling} define a valid
  coupling because $\int_0^T(I-2\bu_s\bu_s^\top\indic(s<\btau))d\bw_s$
  is a Brownian motion. Furthermore, for $t\ge \btau$, we have that
  $\bx_t^{C,1}=\bx_t^{C,2}$. 

  Analogous to \cite{eberle2016reflection}, we aim to construct a function
  $h:[0,D]\to \bbR$  such that $h(0)=0$, $h'(0)=1$, $h'(x)>0$,  and $h''(x) <0$ and
  a constant $a>0$ such that $e^{\eta at}h(\|\bz_t\|)$ is a
  supermartingale. 
  One simplifying assumption for the construction is
  that we only need to define $h$ over the compact set $[0,D]$, while
  \cite{eberle2016reflection} requires $h$ to be defined over
  $[0,\infty)$. This is due to the fact that our solutions will be
  contained in $\cK$ which has diameter $D$, while in
  \cite{eberle2016reflection} the solutions are unconstrained.  
  
  Now we will describe why the construction of such
  an $h$ proves the lemma.
  The supermartingale property will ensure that
  \begin{equation}
    \label{eq:couplingContract}
    \bbE[h(\|\brho_t\|)] \le e^{-\eta at} \bbE[h(\|\brho_0\|)]
  \end{equation}

  Let $W_h$ denote that $1$-Wasserstein distance corresponding to the
  function $d(x,y)=h(\|x-y\|)$ for $x,y\in \cK$.   In other words, if
  $P$ and $Q$ are probability distributions on $\cK$ and $C(P,Q)$ is
  the set of couplings between $P$ and $Q$, then
  \begin{equation*}
    W_h(P,Q)=\inf_{\Gamma\in C(P,Q)} \int_{\cK\times \cK} h(\|x-y\|)d\Gamma(x,y).
  \end{equation*}
  By the hypotheses on $h$,
  $d(x,y)=d(y,x)\ge 0$ and $d(x,y)=0$ if and only if $x=y$. Thus, $W_h$ is
  a valid Wasserstein distance. 

  Assume that $\Gamma_0$ is an optimal coupling of the initial laws
  $C(\cL(\bx_0^{C,1}),\cL(\bx_0^{C,2}))$ so that
  $$
  W_h(\cL(\bx_0^{C,1}),\cL(\bx_0^{C,2})) = \int_{\cK\times
    \cK}h(\|x-y\|)d\Gamma_0(x,y).
  $$
  Such a coupling exists by Theorem
  4.1 of \cite{villani2008optimal}. Then using this initial coupling on the
  right of \eqref{eq:couplingContract} and minimizing over all couplings
  of the dynamics
  on the left shows that
  \begin{equation}
    \label{eq:auxWassersteinContract}
    W_h(\cL(\bx_t^{C,1}),\cL(\bx_t^{C,2}))\le \bbE[h(\|\brho_t\|)] \le
    e^{-\eta at} W_h(\cL(\bx_0^{C,1}),\cL(\bx_0^{C,1})).
  \end{equation}
  In other words, the law of the continuous-time RSDE is
  contractive with respect to $W_h$.

  By the assumptions that $h(0)=0$, $h'(0)=1$, $h'(x) >0$, and
  $h''(x)<0$, we have that for all $x\in [0,D]$:
  \begin{equation*}
    h'(D)x \le h(x) \le x.
  \end{equation*}
  It then follows from the definition of $W_h$ and $W_1$ that for all
  probability measures $P$ and $Q$ over $\cK$ that
  \begin{equation*}
    h'(D) W_1(P,Q)\le W_h(P,Q) \le W_1(P,Q).
  \end{equation*}
  Combining these inequalities with \eqref{eq:auxWassersteinContract}
  gives \eqref{eq:W1contract} with $c_{\ref{W1mult}} = h'(D)^{-1}$.

  Now we will construct $h$. The restriction of $h$ to the domain of
  $[0,D]$, along with the Lipschitz bound on $\nabla_x f$ will
  enable an explicit construction of $h$ as the solution to a simple
  harmonic oscillator problem. This is in contrast to the more
  abstract construction in terms of integrals from \cite{eberle2016reflection}.

  To ensure that $e^{\eta at}h(\|\brho_t\|)$ is a supermartingale, we
  must ensure that this process is non-increasing on average. Recall
  that $\btau$ is the coupling time so that $e^{\eta
    at}h(\|\brho_t\|)=0$ for $t\ge \btau$. So, it suffices to bound the
  behavior of the process for all $t< \btau$. In this case,
  we require
  that non-martingale terms of $d\left( e^{\eta at}h(\|\brho_t\|)\right)$ are
  non-positive.  
  By It\^o's formula we have that
  \begin{equation}
    \label{eq:superDifferential}
    d(e^{\eta at}h(\|\brho_t\|)) = e^{\eta at} \eta ah(\|\brho_t\|)dt +
    e^{\eta at} h'(\|\brho_t\|)d\|\brho_t\| + \frac{1}{2} e^{\eta a t}
    h''(\|\brho_t\|)(d\|\brho_t\|)^2.
  \end{equation}
  Thus, the desired differential is computed from $d\|\brho_t\|$ and
  $(d\|\brho_t\|)^2$. So, our next goal is to derive these terms.
  
  Let $\bb_t$ be the one-dimensional Brownian motion defined by
  $d\bb_t = \bu_t^\top d\bw_t$. Then for $t<\btau$,  $d\brho_t$ can be expressed as
  \begin{multline}
    \label{eq:differenceLin}
    d\brho_t = \eta (\nabla_x f(\bx_t^{C,2},\bz_{\floor*{t}})-\nabla_x f(\bx_t^{C,1},\bz_{\floor*{t}}))dt + \\
    \sqrt{\frac{8\eta }{\beta}} \bu_td\bb_t
    +\bv_t^{C,2}d\bmu^{C,2}(t)-\bv_t^{C,1} d\bmu^{C,1}(t).
  \end{multline}
  Since $\bvarphi_t^1$ and $\bvarphi_t^2$ are bounded variation
  processes, the quadratic terms are given by
  \begin{equation}
    \label{eq:differenceQuad}
    (d\brho_t)(d\brho_t)^\top = \frac{8\eta}{\beta} \bu_t \bu_t^\top dt.
  \end{equation}
  
  If
  $u=\rho/\|\rho\|$ and $\rho\ne 0$, then the gradient and Hessian of $\|\rho\|$ are given by
  \begin{equation}
    \label{eq:normDerivatives}
    \nabla \|\rho\| = \|\rho\|^{-1} \rho = u \quad \textrm{and} \quad \nabla^2 \|\rho\| = \|\rho\|^{-1}
    I - \|\rho\|^{-1}uu^\top.
  \end{equation}
  Plugging \eqref{eq:differenceLin}, \eqref{eq:differenceQuad}, and
  \eqref{eq:normDerivatives} into It\^o's formula and simplifying
  gives
  
  \begin{align}
    \nonumber
    d\|\brho_t\|&= \eta \bu_t^\top (\nabla_x f(\bx_t^{C,2},\bz_{\floor*{t}})-\nabla_x
                f(\bx_t^{C,1},\bz_{\floor*{t}}))dt +\sqrt{\frac{8\eta}{\beta}}d\bb_t  \\
    \nonumber
    & +
                \bu_t^\top \bv_t^{C,2}d\bmu^{C,2}(t) -\bu_t^\top
                \bv_t^{C,1}d\bmu^{C,1}(t) \\
    \label{eq:reflectionNormBound}
    &\le \eta \ell Ddt + \sqrt{\frac{8\eta}{\beta}} d\bb_t.
  \end{align}
  The simplification in the equality arises because $(d\brho_t)^\top
  (\nabla^2 \|\brho_t\|) (d\brho_t)=0$. The inequality uses two
  simplifications. The first term on the right arises due to the
  Lipschitz bound on $\nabla_x f$ and the diameter bound on $\cK$. The
  other terms can be removed since $\bx_t^{C,1}$ and $\bx_t^{C,2}$ are
  both in $\cK$,  so that $\bv_t^2 \in N_{\cK}(\bx_t^{C,2})$
  implies that $(\bx_t^{C,1}-\bx_t^{C,2})^\top \bv_t^2\le 0$. Likewise,
  $\bv_t^1\in N_{\cK}(\bx_t^{C,1})$ implies that
  $-(\bx_t^{C,1}-\bx_t^{C,2})^\top \bv_t^1\le 0$. Then since $\bmu^1$
  and $\bmu^2$ are non-negative measures, the corresponding terms are
  non-positive.

  Note that we also have that $(d\|\brho_t\|)^2=\frac{8\eta}{\beta}
  dt$. Plugging the bounds for $d\|\brho_t\|$ and $(d\|\brho_t\|)^2$ into
  \eqref{eq:superDifferential} gives
  \begin{multline}
    \label{eq:supermartingaleRequirement}
    d(e^{\eta at}h(\|\brho_t\|))\le \frac{4\eta}{\beta} e^{\eta at}
    \left(
      \frac{a\beta}{4} h(\|\brho_t\|) + \frac{D\ell\beta}{4}
      h'(\|\brho_t\|) + h''(\|\brho_t\|)
    \right)dt +\\
    \sqrt{\frac{8\eta}{\beta}}e^{\eta a t} h'(\|\brho_t\|)d\bb_t.
  \end{multline}
  Thus, we see that a sufficient condition for $e^{\eta a
    t}h(\|\brho_t\|)$ to be a supermartingale is that
  \begin{equation}
    \frac{a\beta}{4} h(x) + \frac{D\ell\beta}{4}
    h'(x) + h''(x)=0
  \end{equation}
  for all $x\in [0,D]$. This is precisely the simple harmonic
  oscillator equation for natural frequency and damping ratio defined
  by:
  \begin{equation*}
    \omega_N^2 = \frac{a\beta}{4} \quad\textrm{and} \quad
    2\xi \omega_N = \frac{D\ell\beta}{4}.
  \end{equation*}

  For any positive $a$, the simple harmonic oscillator has a solution with $h(0)=0$, and $h'(0)=1$. Lemma~\ref{lem:oscillator} from Appendix~\ref{sec:oscillator} gives explicit values of $a$ that lead to $h$ with $h'(x) >0$ and  $h''(x) <0$ for all $x\in D$, and gives explicit expressions for $c_{\ref{W1mult}}=(h'(D))$ in these cases. The result follows by plugging in these values. 
\end{proof}

Note that the function $\bar{f}(x)$ satisfies all of the same assumptions that $f(x,z)$ does, with the further property that it is independent of $z$. As a result, Proposition~\ref{prop:W1contract} applies to $\bx_t^M$ as well. We can use this fact to prove the exponential convergence with respect to $W_1$ result from \ref{lem:convergeToStationary}. 

\paragraph{Proof of Lemma~\ref{lem:convergeToStationary}.}
Lemma \ref{lem:invariance} from Appendix~\ref{sec:invariance} implies that  $\pi_{\beta \bar f}$ is invariant with respect to the dynamics of the process $\bx^M$.

Now, apply Proposition~\ref{prop:W1contract} to $\bx^M=\bx^{M,1}$ and $\bx^{M,2}$ such that  
$\cL(\bx_0^{M,2})=\pi_{\beta \bar f}$ to give
$$
W_1(\cL(\bx_t^M),\pi_{\beta \bar f})\le c_{\ref{W1mult}}e^{-\eta a t} W_1(\cL(\bx_0^M),\pi_{\beta \bar f})
\le c_{\ref{W1mult}} D e^{-\eta a t} .
$$
   The specific form from the lemma
 arises because in this case
  $\cL(\bx_t^{M,2})=\pi_{\beta \bar f}$ for all $t\ge 0$, and also that $W_1(\cL(\bx_0^M),\pi_{\beta \bar f})\le D$,
  since $\cK$ has diameter $D$. 
Setting $c_{\ref{globalContract}}=c_{\ref{W1mult}}D$ gives the result.
  \hfill$\blacksquare$

\section{A Switching Argument for Uniform Bounds}

The following lemma, which is based on a method from \cite{chau2019stochastic}, is useful for deriving $W_1$ bounds from $\cL(\bx_t^C)$ that hold uniformly over time. 
It is proved in Appendix~\ref{sec:lemProofs}.

\begin{lemma}
  \label{lem:switch}
  Assume that $\eta \le 1/2$. 
  Let $\hat \bx$ be a process such that for all $0\le s\le t$, if $\hat \bx_s=\bx_s^C$ then $W_1(\cL(\hat \bx_t),\cL(\bx_t^C))\le g(t-s)$, where $g$ is a monotonically increasing function. If $\hat\bx_0=\bx_0^C$, then for all $t\ge 0$, we have that 
  \begin{equation*}
    W(\cL(\hat\bx_t),\cL(\bx_t^C))\le g(\eta^{-1})\left(1+\frac{c_{\ref{W1mult}}}{1-e^{-a/2}}\right)
  \end{equation*}
\end{lemma}\

We will refer to this as the ``switching lemma'', as the proof follows by constructing a sequence of processes that switch from the dynamics of $\hat\bx$ to the dynamics of $\bx^C$. 

\section{Bounding the Algorithm from the Continuous RSDE}
\label{sec:AtoC}

The goal of this section is to prove Lemma~\ref{lem:AtoC}, which states that the law of the algorithm, $\bx_t^A$, is close to the law of the continuous reflected SDE, $\bx_t^C$. 
To derive this bound, we introduce an intermediate process $\bx^D$, and show that its law is close to that of both $\bx^C$ and $\bx^A$.

Recall the process $\by^C$ defined in Subsection~\ref{ss:processes}. For any initial $\bx_0^D\in\cK$, we define the following iteration on the integers:
$$
\bx_{k+1}^D=\Pi_{\cK}(\bx_k^D+\by^C_{k+1}-\by^C_k),
$$
and set $\bx_t^D=\bx_{\floor*{t}}^D$ for all $t\in \bbR$.

The process, $\bx^D$, can also be interpreted as a Skorokhod solution. Indeed, let $\cD$ be the discretization operator that sets $\cD(x)_t = x_{\floor*{t}}$ for any continuous-time trajectory, $x_t$. Then, provided that $\bx_0^D=\bx_0^C$, we have that $\bx^D=\cS(\cD(\by^C))$. Recall that $\cS$ corresponds to the Skorokhod solution. See Appendix~\ref{appsec:skorohod} for a more detailed explanation of this construction.  

The following lemmas give the specific bounds on the differences between $\cL(\bx_t^C)$ and $\cL(\bx_t^D)$, and between $\cL(\bx_t^A)$ and $\cL(\bx_t^D)$, respectively. They are proved in Appendix~\ref{sec:lemProofs}.

\begin{lemma}
\label{lem:tanakaMean}
\const{tanakaRt}
\const{tanakaConst}
Assume that $\bx_0^D=\bx_0^C$ and $\eta\le 1$. 
There are constants, $c_{\ref{tanakaRt}}$ and $c_{\ref{tanakaConst}}$ such that for all  $t\ge 0$, the following bound holds:
\begin{equation*}
W_1(\cL(\bx_t^C),\cL(\bx_t^D))\le  \bbE\left[\|\bx_t^C-\bx_t^D\| \right] \le \left( \eta \log(4\max\{1,t\})\right)^{1/4}
  \left( c_{\ref{tanakaRt}} \sqrt{\eta t} + c_{\ref{tanakaConst}} \right)
\end{equation*}
  The constants are given by:
  \begin{align*}
    c_{\ref{tanakaRt}}  &=  \sqrt{2 \left( \frac{u+\ell D}{2r} +\frac{n \sigma}{\sqrt{2}r} +\frac{2n\sqrt{2}}{r\sqrt{\beta}}\right)
                          \left(
        \frac{n}{\beta} + Du + 2Dn\sigma
        \right)}
                       \\
    c_{\ref{tanakaConst}} &=
                            \sqrt{2 \left(Du + 2n\sigma + \frac{n}{\beta} \right)} + D \sqrt{\frac{u+\ell D}{2r} +\frac{n \sigma}{\sqrt{2}r} +\frac{2n\sqrt{2}}{r\sqrt{\beta} } }
  \end{align*}
\end{lemma}

\begin{lemma}
  \label{lem:linUpper}
  Assume that $\bx_0^A=\bx_0^D$ and $\eta \le 1$. Then for all $t\ge 0$:
  $$
  W_1(\cL(\bx_t^A),\cL(\bx_t^D)) \le
\left( \eta \log(4\max\{1,t\})\right)^{1/4} 
      \left( c_{\ref{tanakaRt}} \sqrt{\eta t} + c_{\ref{tanakaConst}} \right) ((1+\eta\ell)^t -1)
  $$
\end{lemma}

Now Lemma~\ref{lem:AtoC} can be proved by combining Lemmas~\ref{lem:switch}, \ref{lem:tanakaMean}, and \ref{lem:linUpper}.

\paragraph{Proof of Lemma~\ref{lem:AtoC}}
  Using the triangle inequality and Lemmas~
  \ref{lem:tanakaMean} and \ref{lem:linUpper} gives
  \begin{align*}
    W_1(\cL(\bx_t^A),\cL(\bx_t^C))
    &\le
      W_1(\cL(\bx_t^A),\cL(\bx_t^D)) + W(\cL(\bx_t^D),\cL(\bx_t^C)) \\
    &\le \left( \eta \log(4\max\{1,t\})\right)^{1/4} 
      \left( c_{\ref{tanakaRt}} \sqrt{\eta t} + c_{\ref{tanakaConst}} \right) (1+\eta\ell)^t.
  \end{align*}
  
  Now we will utilize the switching trick from Lemma~\ref{lem:switch} to simplify the bound. Define $g:[0,t]\to \bbR$ by
  $g(s) =  \left( \eta \log(4\max\{1,t\})\right)^{1/4} 
  \left( c_{\ref{tanakaRt}} \sqrt{\eta s} + c_{\ref{tanakaConst}} \right) (1+\eta\ell)^s.
  $
  Then applying Lemma~\ref{lem:switch} using the bound from $g$ gives the desired bound:
  \begin{align*}
    \MoveEqLeft
    W_1(\cL(\bx_t^A),\cL(\bx_t^C))
    \\
    &\le \left( \eta \log(4\max\{1,t\})\right)^{1/4} 
  \left( c_{\ref{tanakaRt}} + c_{\ref{tanakaConst}} \right) (1+\eta\ell)^{1/\eta} \left(1+\frac{c_{\ref{W1mult}}}{1-e^{-a/2}}\right) 
    \\
    &\le  \left( \eta \log(4\max\{1,t\})\right)^{1/4} 
  \left( c_{\ref{tanakaRt}} + c_{\ref{tanakaConst}} \right) e^\ell \left(1+\frac{c_{\ref{W1mult}}}{1-e^{-a/2}}\right) 
  \end{align*}
  The second inequality uses the fact that for all $\eta >0$,
  \begin{align*}
    (1+\eta\ell)^{1/\eta} \le e^\ell & \iff \frac{\log(1+\eta \ell)}{\eta} \le  \ell
  \end{align*}
  where the right inequality holds due to concavity of the logarithm. 

  Now, for $t\ge 4$ we have $\log(4t)\le 2\log(t)$. So, setting
  $$
  c_{\ref{AtoC}}=2^{1/4} \left( c_{\ref{tanakaRt}} + c_{\ref{tanakaConst}} \right) e^\ell \left(1+\frac{c_{\ref{W1mult}}}{1-e^{-a/2}}\right) 
  $$
  gives the bound $W_1(\cL(\bx_t^A),\cL(\bx_t^C))\le c_{\ref{AtoC}} (\eta \log t)^{1/4}$.
\hfill$\blacksquare$

\section{Averaging Out the External Variables}
\label{sec:averaging}

Now we show that the dynamics of the continuous reflected SDE, $\bx^C$, and its averaged version, $\bx^M$, have similar laws. In particular, we will prove Lemma~\ref{lem:CtoM}. The general strategy is similar to that of Section~\ref{sec:AtoC}. Namely, we devise a new process, $\bx^B$ that fits ``between'' $\bx^C$ and $\bx^M$. Then the desired bound is given by showing that $\cL(\bx_t^M)$ is close to $\cL(\bx_t^B)$ and that $\cL(\bx_t^B)$ is close to $\cL(\bx_t^C)$.

The new process is defined by $\bx^B=\cS(\by^B)$ where
\begin{equation}
  \label{eq:Ybetween}
  \by_t^B = \bx_0^B + \sqrt{\frac{2\eta}{\beta}} \bw_t - \eta \int_0^t \nabla_x
  f(\bx^M_s,\bz_{\floor*{s}}) ds.
\end{equation}
So, we see that $\bx^B$ has similar dynamics to $\bx^C$, but $\bx^M$ is used in place of  $\bx^C$ in the drift term.

The lemmas describing the relations between $\cL(\bx_t^M)$ and $\cL(\bx_t^B)$ and between $\cL(\bx_t^C)$ and $\cL(\bx_t^B)$ are stated below. They are proved in Appendix~\ref{sec:lemProofs}.

\begin{lemma}
  \label{lem:dudley}
  \const{aveTanakaLin}
  \const{aveTanakaRoot}
  \const{aveTanakaTQ}
  Assume that $\bx_0^M = \bx_0^B$ and that $\eta \le 1$. Then is a  positive constants, $c_{\ref{aveTanakaLin}}$, $c_{\ref{aveTanakaRoot}}$, abd $c_{\ref{aveTanakaTQ}}$ such that for all $t\ge 0$,
  \begin{equation*}
W_1(\cL(\bx_t^B),\cL(\bx_t^M))\le    \bbE\left[\|\bx_t^B-\bx_t^M\| \right] \le c_{\ref{aveTanakaLin}} \eta t^{1/2} + c_{\ref{aveTanakaRoot}} \eta^{1/2} t^{1/4} + c_{\ref{aveTanakaTQ}} \eta t^{3/4}
  \end{equation*}
  The constants are given by:
  \begin{align*}
    \nonumber
    c_{\ref{aveTanakaLin}} &=
                             2\sigma \sqrt{n} \\
    c_{\ref{aveTanakaRoot}} &=
\sqrt{\frac{64 n \sigma D\sqrt{2\pi} }{r} } \\
    c_{\ref{aveTanakaTQ}} &=
\sqrt{\frac{128 n \sigma\sqrt{2\pi}}{r} \left(
        \frac{n}{\beta} + Du + 2Dn\sigma
    \right)} 
  \end{align*}
\end{lemma}

\begin{lemma}
  \label{lem:gronwall}
  Assume that $\bx_0^C=\bx_0^B$  and $\eta \le 1$. Then for all $t\ge 0$,
  \begin{equation*}
    W_1(\cL(\bx_t^B),\cL(\bx_t^C)) \le \left( c_{\ref{aveTanakaLin}} \eta t^{1/2} + c_{\ref{aveTanakaRoot}} \eta^{1/2} t^{1/4} + c_{\ref{aveTanakaTQ}} \eta t^{3/4}\right) (e^{\eta \ell t}-1).
  \end{equation*}
  \end{lemma}

  \paragraph*{Proof of Lemma~\ref{lem:CtoM}}
    Using the triangle inequality along with Lemmas~\ref{lem:dudley} and \ref{lem:gronwall} shows that
  \begin{align*}
    W_1(\cL(\bx_t^M),\cL(\bx_t^C))
    &\le
      W_1(\cL(\bx_t^M),\cL(\bx_t^B)) + W_1(\cL(\bx_t^B),\cL(\bx_t^C)) \\
    &\le \left( c_{\ref{aveTanakaLin}} \eta t^{1/2} + c_{\ref{aveTanakaRoot}} \eta^{1/2} t^{1/4}+c_{\ref{aveTanakaTQ}}\eta t^{3/4}\right) e^{\eta \ell t}
  \end{align*}
  Using Lemma~\ref{lem:switch} along with the fact that $\eta^{1/2}\le \eta^{1/4}$ gives $W_1(\cL(\bx_t^M),\cL(\bx_t^C))\le c_{\ref{CtoM}} \eta ^{1/4}$ with
  $$
  c_{\ref{CtoM}} =  (c_{\ref{aveTanakaLin}} + c_{\ref{aveTanakaRoot}} + c_{\ref{aveTanakaTQ}}) e^{\ell} \left(1+\frac{c_{\ref{W1mult}}}{1-e^{-a/2}}\right).
  $$
\hfill$\blacksquare$

\section{Conclusions and Future Work}
    In this paper, we have given non-asymptotic bounds for a projected stochastic gradient Langevin algorithm applied to non-convex functions with IID external random variables. In particular, we demonstrated convergence of sampling methods with respect to the $1$-Wasserstein distance and showed how the sampling results can be utilized for non-convex learning. The results were derived using a novel approach contraction analysis for reflected SDEs. The contraction analysis utilizes connections with simple harmonic oscillator problems to get explicit contraction rate bounds. Future work will include a variety of extensions. The assumption of a compact convex domain, $\cK$, can likely be relaxed to a general class of non-convex non-compact domains. This would require use of more general analysis of Skorokhod problems as in \cite{lions1984stochastic} along with a dissipativity condition to ensure that the reflected SDEs remain contractive. The assumptions that $\bz_k$ are IID and that $\nabla_x f(x,\bz_k)$ is sub-Gaussian will also be relaxed in future work. The eventual goal will be to use the method for problems in time-series analysis, control, and reinforcement learning. 

    \acks{The author would like to thank Tyler Lekang, Jonah Roux, Suneel Sheikh, Chuck Hisamoto, and Michael Schmit for helpful discussions. The author acknowledges funding from NASA STTR 19-1-T4.03-3451 and NSF CMMI 1727096}

\if\ARXIV0 
\newpage
\fi
    
\bibliography{cool-refs}

\begin{thebibliography}{40}
\providecommand{\natexlab}[1]{#1}
\providecommand{\url}[1]{\texttt{#1}}
\expandafter\ifx\csname urlstyle\endcsname\relax
  \providecommand{\doi}[1]{doi: #1}\else
  \providecommand{\doi}{doi: \begingroup \urlstyle{rm}\Url}\fi

\bibitem[Ahn and Chewi(2020)]{ahn2020efficient}
Kwangjun Ahn and Sinho Chewi.
\newblock Efficient constrained sampling via the mirror-langevin algorithm.
\newblock \emph{arXiv preprint arXiv:2010.16212}, 2020.

\bibitem[Borkar and Mitter(1999)]{borkar1999strong}
Vivek~S Borkar and Sanjoy~K Mitter.
\newblock A strong approximation theorem for stochastic recursive algorithms.
\newblock \emph{Journal of optimization theory and applications}, 100\penalty0
  (3):\penalty0 499--513, 1999.

\bibitem[Brosse et~al.(2017)Brosse, Durmus, Moulines, and
  Pereyra]{brosse2017sampling}
Nicolas Brosse, Alain Durmus, {\'E}ric Moulines, and Marcelo Pereyra.
\newblock Sampling from a log-concave distribution with compact support with
  proximal langevin monte carlo.
\newblock \emph{arXiv preprint arXiv:1705.08964}, 2017.

\bibitem[Bubeck et~al.(2015)Bubeck, Eldan, and Lehec]{bubeck2015finite}
Sebastien Bubeck, Ronen Eldan, and Joseph Lehec.
\newblock Finite-time analysis of projected langevin monte carlo.
\newblock \emph{Advances in Neural Information Processing Systems},
  28:\penalty0 1243--1251, 2015.

\bibitem[Bubeck et~al.(2018)Bubeck, Eldan, and Lehec]{bubeck2018sampling}
S{\'e}bastien Bubeck, Ronen Eldan, and Joseph Lehec.
\newblock Sampling from a log-concave distribution with projected langevin
  monte carlo.
\newblock \emph{Discrete \& Computational Geometry}, 59\penalty0 (4):\penalty0
  757--783, 2018.

\bibitem[Chau et~al.(2019)Chau, Moulines, R{\'a}sonyi, Sabanis, and
  Zhang]{chau2019stochastic}
Ngoc~Huy Chau, {\'E}ric Moulines, Miklos R{\'a}sonyi, Sotirios Sabanis, and
  Ying Zhang.
\newblock On stochastic gradient langevin dynamics with dependent data streams:
  the fully non-convex case.
\newblock \emph{arXiv preprint arXiv:1905.13142}, 2019.

\bibitem[Chen et~al.(2020)Chen, Du, and Tong]{chen2020stationary}
Xi~Chen, Simon~S Du, and Xin~T Tong.
\newblock On stationary-point hitting time and ergodicity of stochastic
  gradient langevin dynamics.
\newblock \emph{Journal of Machine Learning Research}, 21\penalty0
  (68):\penalty0 1--41, 2020.

\bibitem[Cheng et~al.(2018)Cheng, Chatterji, Abbasi-Yadkori, Bartlett, and
  Jordan]{cheng2018sharp}
Xiang Cheng, Niladri~S Chatterji, Yasin Abbasi-Yadkori, Peter~L Bartlett, and
  Michael~I Jordan.
\newblock Sharp convergence rates for langevin dynamics in the nonconvex
  setting.
\newblock \emph{arXiv preprint arXiv:1805.01648}, 2018.

\bibitem[Coffey and Kalmykov(2012)]{coffey2012langevin}
William Coffey and Yu~P Kalmykov.
\newblock \emph{The Langevin equation: with applications to stochastic problems
  in physics, chemistry and electrical engineering}, volume~27.
\newblock World Scientific, 2012.

\bibitem[Cover and Thomas(2012)]{cover2012elements}
Thomas~M Cover and Joy~A Thomas.
\newblock \emph{Elements of information theory}.
\newblock John Wiley \& Sons, 2012.

\bibitem[Durmus et~al.(2017)Durmus, Moulines, et~al.]{durmus2017nonasymptotic}
Alain Durmus, Eric Moulines, et~al.
\newblock Nonasymptotic convergence analysis for the unadjusted langevin
  algorithm.
\newblock \emph{The Annals of Applied Probability}, 27\penalty0 (3):\penalty0
  1551--1587, 2017.

\bibitem[Eberle(2016)]{eberle2016reflection}
Andreas Eberle.
\newblock Reflection couplings and contraction rates for diffusions.
\newblock \emph{Probability theory and related fields}, 166\penalty0
  (3-4):\penalty0 851--886, 2016.

\bibitem[Erdogdu et~al.(2018)Erdogdu, Mackey, and Shamir]{erdogdu2018global}
Murat~A Erdogdu, Lester Mackey, and Ohad Shamir.
\newblock Global non-convex optimization with discretized diffusions.
\newblock In \emph{Advances in Neural Information Processing Systems}, pages
  9671--9680, 2018.

\bibitem[Fehrman et~al.(2020)Fehrman, Gess, and
  Jentzen]{fehrman2020convergence}
Benjamin Fehrman, Benjamin Gess, and Arnulf Jentzen.
\newblock Convergence rates for the stochastic gradient descent method for
  non-convex objective functions.
\newblock \emph{Journal of Machine Learning Research}, 21, 2020.

\bibitem[Gelfand and Mitter(1991)]{gelfand1991recursive}
Saul~B Gelfand and Sanjoy~K Mitter.
\newblock Recursive stochastic algorithms for global optimization in r\^{}d.
\newblock \emph{SIAM Journal on Control and Optimization}, 29\penalty0
  (5):\penalty0 999--1018, 1991.

\bibitem[Gilbarg and Trudinger(1998)]{gilbarg1998elliptic}
David Gilbarg and Neil~S Trudinger.
\newblock \emph{Elliptic partial differential equations of second order}.
\newblock Springer, 1998.

\bibitem[Gray(2011)]{gray2011entropy}
Robert~M Gray.
\newblock \emph{Entropy and information theory}.
\newblock Springer Science \& Business Media, 2011.

\bibitem[Harrison and Williams(1987)]{harrison1987multidimensional}
J~Michael Harrison and Ruth~J Williams.
\newblock Multidimensional reflected brownian motions having exponential
  stationary distributions.
\newblock \emph{The Annals of Probability}, pages 115--137, 1987.

\bibitem[Herbster and Warmuth(2001)]{herbster2001tracking}
Mark Herbster and Manfred~K Warmuth.
\newblock Tracking the best linear predictor.
\newblock \emph{Journal of Machine Learning Research}, 1\penalty0
  (Sep):\penalty0 281--309, 2001.

\bibitem[Hsieh et~al.(2018)Hsieh, Kavis, Rolland, and
  Cevher]{hsieh2018mirrored}
Ya-Ping Hsieh, Ali Kavis, Paul Rolland, and Volkan Cevher.
\newblock Mirrored langevin dynamics.
\newblock \emph{Advances in Neural Information Processing Systems},
  31:\penalty0 2878--2887, 2018.

\bibitem[Kallenberg(2002)]{kallenberg2002foundations}
Olav Kallenberg.
\newblock \emph{Foundations of Modern Probability}.
\newblock Springer Science \& Business Media, 2002.

\bibitem[Krichene and Bartlett(2017)]{krichene2017acceleration}
Walid Krichene and Peter~L Bartlett.
\newblock Acceleration and averaging in stochastic mirror descent dynamics.
\newblock \emph{arXiv preprint arXiv:1707.06219}, 2017.

\bibitem[Lattimore and Szepesv{\'a}ri(2019)]{lattimore2019bandit}
Tor Lattimore and Csaba Szepesv{\'a}ri.
\newblock Bandit algorithms.
\newblock \emph{Cambridge University Press (preprint)}, 2019.

\bibitem[Lee(2013)]{lee2013introduction}
John~M Lee.
\newblock \emph{Introduction to Smooth Manifolds}.
\newblock Springer, 2013.

\bibitem[Lions and Sznitman(1984)]{lions1984stochastic}
Pierre-Louis Lions and Alain-Sol Sznitman.
\newblock Stochastic differential equations with reflecting boundary
  conditions.
\newblock \emph{Communications on Pure and Applied Mathematics}, 37\penalty0
  (4):\penalty0 511--537, 1984.

\bibitem[Ma et~al.(2019)Ma, Chen, Jin, Flammarion, and Jordan]{ma2019sampling}
Yi-An Ma, Yuansi Chen, Chi Jin, Nicolas Flammarion, and Michael~I Jordan.
\newblock Sampling can be faster than optimization.
\newblock \emph{Proceedings of the National Academy of Sciences}, 116\penalty0
  (42):\penalty0 20881--20885, 2019.

\bibitem[Majka et~al.(2020)Majka, Mijatovi{\'c}, Szpruch,
  et~al.]{majka2020nonasymptotic}
Mateusz~B Majka, Aleksandar Mijatovi{\'c}, {\L}ukasz Szpruch, et~al.
\newblock Nonasymptotic bounds for sampling algorithms without log-concavity.
\newblock \emph{Annals of Applied Probability}, 30\penalty0 (4):\penalty0
  1534--1581, 2020.

\bibitem[Nesterov and Nemirovskii(1994)]{nesterov1994interior}
Yurii Nesterov and Arkadii Nemirovskii.
\newblock \emph{Interior-point polynomial algorithms in convex programming}.
\newblock SIAM, 1994.

\bibitem[Raginsky et~al.(2017)Raginsky, Rakhlin, and
  Telgarsky]{raginsky2017non}
Maxim Raginsky, Alexander Rakhlin, and Matus Telgarsky.
\newblock Non-convex learning via stochastic gradient langevin dynamics: a
  nonasymptotic analysis.
\newblock \emph{arXiv preprint arXiv:1702.03849}, 2017.

\bibitem[Roberts et~al.(1996)Roberts, Tweedie, et~al.]{roberts1996exponential}
Gareth~O Roberts, Richard~L Tweedie, et~al.
\newblock Exponential convergence of langevin distributions and their discrete
  approximations.
\newblock \emph{Bernoulli}, 2\penalty0 (4):\penalty0 341--363, 1996.

\bibitem[Rockafellar(2015)]{rockafellar2015convex}
Ralph~Tyrell Rockafellar.
\newblock \emph{Convex Analysis}, volume~36.
\newblock Princeton University Press, 2015.

\bibitem[S{\l}omi{\'n}ski(2001)]{slominski2001euler}
Leszek S{\l}omi{\'n}ski.
\newblock Euler's approximations of solutions of sdes with reflecting boundary.
\newblock \emph{Stochastic processes and their applications}, 94\penalty0
  (2):\penalty0 317--337, 2001.

\bibitem[Tanaka et~al.(1979)]{tanaka1979stochastic}
Hiroshi Tanaka et~al.
\newblock Stochastic differential equations with reflecting boundary condition
  in convex regions.
\newblock \emph{Hiroshima Mathematical Journal}, 9\penalty0 (1):\penalty0
  163--177, 1979.

\bibitem[Vershynin(2018)]{vershynin2018high}
Roman Vershynin.
\newblock \emph{High-dimensional probability: An introduction with applications
  in data science}, volume~47.
\newblock Cambridge university press, 2018.

\bibitem[Villani(2008)]{villani2008optimal}
C{\'e}dric Villani.
\newblock \emph{Optimal transport: old and new}, volume 338.
\newblock Springer Science \& Business Media, 2008.

\bibitem[Wainwright(2019)]{wainwright2019high}
Martin~J Wainwright.
\newblock \emph{High-dimensional statistics: A non-asymptotic viewpoint},
  volume~48.
\newblock Cambridge University Press, 2019.

\bibitem[Wang et~al.(2020)Wang, Lei, and Panageas]{wang2020fast}
Xiao Wang, Qi~Lei, and Ioannis Panageas.
\newblock Fast convergence of langevin dynamics on manifold: Geodesics meet
  log-sobolev.
\newblock \emph{Advances in Neural Information Processing Systems}, 33, 2020.

\bibitem[Welling and Teh(2011)]{welling2011bayesian}
Max Welling and Yee~W Teh.
\newblock Bayesian learning via stochastic gradient langevin dynamics.
\newblock In \emph{Proceedings of the 28th international conference on machine
  learning (ICML-11)}, pages 681--688, 2011.

\bibitem[Xu et~al.(2018)Xu, Chen, Zou, and Gu]{xu2018global}
Pan Xu, Jinghui Chen, Difan Zou, and Quanquan Gu.
\newblock Global convergence of langevin dynamics based algorithms for
  nonconvex optimization.
\newblock In \emph{Advances in Neural Information Processing Systems}, pages
  3126--3137, 2018.

\bibitem[Zhang et~al.(2020)Zhang, Peyr{\'e}, Fadili, and
  Pereyra]{zhang2020wasserstein}
Kelvin~Shuangjian Zhang, Gabriel Peyr{\'e}, Jalal Fadili, and Marcelo Pereyra.
\newblock Wasserstein control of mirror langevin monte carlo.
\newblock \emph{arXiv preprint arXiv:2002.04363}, 2020.

\end{thebibliography}
\appendix

\section{Bounds on Simple Harmonic Oscillator Coefficients.}
\label{sec:oscillator}
\begin{lemma}
   \label{lem:oscillator}
   Consider the simple harmonic oscillator
   $$
    \omega_N^2 h(x) + 2\xi\omega_N
    h'(x) + h''(x)=0
    $$
    with
  \begin{equation*}
    \omega_N = \frac{\sqrt{a\beta}}{2} \quad \textrm{and} \quad
    \xi = \frac{D\ell}{4}\sqrt{\frac{\beta}{a}}.
  \end{equation*}
  and boundary condition $h(0)=0$ and $h'(0)=1$. 

  For any positive values of $D$, $\ell$, and $\beta$ if $a$ is set to 
  $$
    a=\frac{D^2\ell^2\beta}{16}\left(1-\tanh^2\left(
       \frac{D^2\ell\beta}{8}\right)\right)
  $$
  then $h'(x) >0$ and $h''(x)<0$ for all $x\in[0,D]$ and
  $$
  (h'(D))^{-1} = \frac{e^{D\omega_N \xi}}{\cosh(D\omega_N\sqrt{\xi^2-1})-\frac{\xi}{\sqrt{\xi^2-1}}\sinh(D\omega_N\sqrt{1-\xi^2})}
  $$

  If $D^2\ell \beta <8$, then $a$ can be set to $a=\frac{4}{D^2\beta}$ and in this case $h'(x) >0$ and $h''(x)<0$ for all $x\in[0,D]$ and
  $$
  (h'(D))^{-1} =
\frac{e^{D\omega_N \xi}}{\cos(D\omega_N\sqrt{1-\xi^2})-\frac{\xi}{\sqrt{1-\xi^2}}\sin(D\omega_N\sqrt{1-\xi^2})}.
  $$

\end{lemma}

\begin{proof}
We will tune $a$ to ensure that $h'(x) >0$ for all
  $x\in [0,D]$. Then since $h(x)\ge 0$ for $x\in [0,D]$, the simple
  harmonic oscillator equation implies that $h''(x) <0 $ for $x\in
  [0,D]$.

  We will consider the underdamped case with $\xi < 1$ and the
  overdamped case with $\xi > 1$. We will see that for any collection
  of parameters, $a$ can be chosen to give an overdamped solution with
  the desired properties. However, when $D^2\ell \beta < 8$, a larger
  $a$ can be chosen which gives rise to an underdamped solution with
  the desired properties.

  First we consider the underdamped case. The expression for $\xi$
  from \eqref{eq:oscillatorConstants} shows
  that 
  \begin{equation}
    \label{eq:underdampedLB}
    \xi^2 < 1 \iff \frac{D^2\ell^2 \beta}{16} < a.
  \end{equation}
  Now we will try to maximize $a$ while ensuring that
  $h'(x)>0$. Standard methods from linear differential equations show
  that $h$ and its derivative are given by:
  \begin{align*}
    h(x) &= e^{-x\omega_N \xi
           }\frac{\sin(x\omega_N\sqrt{1-\xi^2})}{\omega_N\sqrt{1-\xi^2}}
    \\
    h'(x)&=
           \frac{e^{-x\omega_N\xi}}{\sqrt{1-\xi^2}}(\sqrt{1-\xi^2}\cos(x\omega_N
           \sqrt{1-\xi^2})-\xi \sin(x\omega_N\sqrt{1-\xi^2})).
  \end{align*}
  The smallest $x>0$ such that  $h'(x)=0$ is the smallest
  $x>0$ such
  that $$(\cos(x\omega_N\sqrt{1-\xi^2}),\sin(x\omega_N\sqrt{1-\xi^2}))=(\xi,\sqrt{1-\xi^2}).$$

  Using the fact that $\sin'(\theta) < 1$ for $\theta\ne 2\pi k$, we
  have that
  \begin{equation*}
    \sin(D\omega_N \sqrt{1-\xi^2})< D \omega_N \sqrt{1-\xi^2}.
  \end{equation*}
  So, if we choose $\omega_N \le D^{-1}$ we will have $\sin(x\omega_N
  \sqrt{1-\xi^2})< \sqrt{1-\xi^2}$ and thus $h''(x) > 0$ for all $x\in
  [0,D]$. Plugging in the expression for $\omega_N$ from
  \eqref{eq:oscillatorConstants} shows that $a$ must satisfy
  \begin{equation}
    \label{eq:underdampedUB}
    a \le \frac{4}{D^2\beta}.
  \end{equation}
  Comparing \eqref{eq:underdampedLB} and \eqref{eq:underdampedUB}
  shows that a suitable $a$ can only be chosen when $D^2\ell\beta <
  8$. The $a$ from the lemma statement is chosen by taking the largest
  possible value. Note that by construction, $a$
  satisfies \eqref{eq:underdampedLB} and so $\xi < 1$ in this case. 

  Now we consider the overdamped case, so that $\xi^2 > 1$. In this
  case, standard methods from linear differential equations show that
  $h$ and its derivative are given by:
  \begin{align*}
    h(x) &= e^{-x\omega_N
           \xi}\frac{\sinh(x\omega_N\sqrt{\xi^2-1})}{\omega_N
           \sqrt{\xi^2-1}} \\
    h'(x)&= \frac{e^{-x\omega_N\xi}}{\sqrt{\xi^2-1}}
           (\sqrt{\xi^2-1}\cosh(x\omega_N\sqrt{\xi^2-1})-\xi\sinh(x\omega_N\sqrt{\xi^2-1}))
  \end{align*}
  Thus $h'(x)=0$ precisely when
  $\tanh(x\omega_N\sqrt{\xi^2-1})=\frac{\sqrt{\xi^2-1}}{\xi}$. Since
  $\tanh$ is monotonically increasing, if
  $\tanh(D\omega_N\sqrt{\xi^2-1}) < \frac{\sqrt{\xi^2-1}}{\xi}$, then
  we will have that $h'(x)>0$ for all $x\in[0,D]$.

  Plugging in the expressions for $\omega_N$ and $\xi$ gives for all $a>0$
  \begin{equation*}
    \tanh(D\omega_N\sqrt{\xi^2-1})=
    \tanh\left(\frac{D}{2}
      \sqrt{\frac{D^2\ell^2\beta^2}{16}-a\beta}\right) < \tanh\left(
      \frac{D^2\ell\beta}{8}
      \right).
    \end{equation*}
    So to ensure that $h'(x)>0$ for all $x\in[0,D]$, it suffices to choose $a$ so that $\frac{\sqrt{\xi^2-1}}{\xi}$ achieves the
    bound on the right. In particular, after some algebra we find that
    \begin{equation*}
      a = \frac{D^2\ell^2\beta}{16}\left(1-\tanh^2\left( \frac{D^2\ell\beta}{8}\right)\right)>0.
    \end{equation*}
    Plugging this expression into the definition of $\xi$ shows that 
     $\xi^2 >1$, and so the oscillator is indeed overdamped. Thus, $h$
     is well-defined and
     has all the desired properties, so the proof is complete. 
   \end{proof}

\section{Proofs of Supporting Lemmas}

\label{sec:lemProofs}

The proofs below use the following notation from \cite{rockafellar2015convex}. Let $\gamma(x|\cK)$ denote the gauge function:
\begin{equation*}
  \gamma(x|\cK) = \inf \{t >0 | x \in t\cK\}
\end{equation*}
and let $\gamma^*(x|\cK)$ be the support function:
\begin{equation*}
  \delta^\star(x|\cK) = \sup\{y^\top x | y\in\cK\}.
\end{equation*}
By the assumption on $\cK$, it follows that $\gamma(x|\cK)\le r^{-1}
\|x\|$.

\paragraph{Proof of Lemma~\ref{lem:switch}}
  Consider the family of switching processes $\hat \bx_{s,t}^C$ be the process such that $\hat \bx_{s,t}^C=\hat \bx_t$ for $t\le s$ and then for $t\ge s$, the dynamics of $\hat \bx_{s,t}^C$ follow (\ref{eq:continuousProjectedLangevin}), the definition of $\bx_t^C$.

  Let $H=\floor*{1/\eta}$ and assume that $t\in [kH,(k+1)H)$. It follows that $\hat \bx_{0,t}^C=\bx_t^C$ and $\hat \bx_{(k+1)H,t}^C=\hat \bx_t$. The triangle inequality then implies that
  \begin{align*}
    W_1(\cL(\bx_t^C),\cL(\hat\bx_t))
    &\le
    \sum_{i=0}^{k} W_1(\cL(\hat\bx_{iH,t}^C),\cL(\hat\bx_{(i+1)H,t}^C)) 
  \end{align*}

  For $i< k$, using Proposition \ref{prop:W1contract}, followed by the hypothesis gives
  \begin{align*}
    W_1(\cL(\hat\bx_{iH,t}^C),\cL(\hat\bx_{(i+1)H,t}^C))
    &\le c_{\ref{W1mult}}e^{-\eta a (t-(i+1)H)}W(\cL(\hat\bx_{iH,(i+1)H}^C),\cL(\hat \bx_{(i+1)H})) \\
    &\le c_{\ref{W1mult}}e^{-\eta a (t-(i+1)H)} g(H) \\
    &\le c_{\ref{W1mult}}e^{- a (k-i-1)/2}g(\eta^{-1})
  \end{align*}
  The final inequality uses the facts that $\frac{1}{2}\le \eta H \le 1$ along with monotonicity of $g$. The lower bound on $\eta H$ arises because $H\ge \eta^{-1}-1$ and so $\eta H \ge 1-\eta \ge 1/2$, since $\eta \le 1/2$. 

  It follows that the first $k$ terms of the sum can be bounded by:
  \begin{align*}
    \sum_{i=0}^{k-1} W_1(\cL(\hat\bx_{iH,t}^C),\cL(\hat\bx_{(i+1)H,t}^C))
    &\le \sum_{i=0}^{k-1}c_{\ref{W1mult}}e^{- a (k-i-1)/2}g(1) \\
    &\le \frac{c_{\ref{W1mult}}g(1)}{1-e^{-a/2}}
  \end{align*}
  
  For $i=k$ the hypothesis gives
  \begin{align*}
    W_1(\cL(\hat\bx_{iH,t}^C),\cL(\hat\bx_{(i+1)H,t}^C))
    &=
     W_1(\cL(\hat\bx_{kH,t}^C),\cL(\hat\bx_{t}))  
    \\
    &\le g(t-kH) \le g(\eta^{-1})
  \end{align*}
  Adding this to the bound from the first $k$ terms gives the result.
  \hfill$\blacksquare$

  \paragraph*{Proof of Lemma~\ref{lem:tanakaMean}}
    The basic idea follows arguments from \cite{bubeck2018sampling}. However, we must deviate from the method to account for the extra randomness due to $\bz_i$. 
  First note that since $\bx_{t}^D=\bx_{\floor*{t}}^D$, we get the following triangle inequality bound:
  \begin{align}
    \label{eq:tailSplit}
    \|\bx_t^C-\bx_t^D\|
    &=\|\bx_t^C-\bx_{\floor*{t}}^C+\bx_{\floor*{t}}^C-\bx_{\floor*{t}}^D\| \\
    &\le \|\bx_t^C-\bx_{\floor*{t}}^C\| +  \|\bx_{\floor*{t}}^C-\bx_{\floor*{t}}^D\|.
  \end{align}

  For simpler notation, set $\floor*{t}=k$.
  
  The first term can be estimated directly via It\^o's rule:
  \begin{multline}
    \label{eq:itoCS}
    d\|\bx_t^C-\bx_{k}^C\|^2
    \\
    =2(\bx_t^C-\bx_k^C)^\top \left(-\eta \nabla_x f(\bx_t^C,\bz_{\floor*{t}})dt -\bv_s d\bmu(s)
      +\sqrt{\frac{2\eta}{\beta}}d\bw_t
      \right) + \frac{2\eta n}{\beta}dt. 
  \end{multline}

  Note that the inner products between the state difference and gradient terms can be bounded by:
  \begin{align}
    \nonumber
    \MoveEqLeft
    (\bx_t^C-\bx_k^C)^\top \nabla_x \bar f(\bx_t^C) + 
    (\bx_t^C-\bx_k^C)^\top \left(\nabla_x f(\bx_t^C,\bz_{\floor*{t}})
    -\nabla_x \bar f(\bx_t^C)  \right) \\
    \label{eq:gradTriangle}
    &\le Du + D \|\nabla_x f(\bx_t^C,\bz_{\floor*{t}})
    -\nabla_x \bar f(\bx_t^C)\|
  \end{align}
  The uniform sub-Gaussian assumption implies that the mean of the term on the right is bounded above by $2n\sigma $. Indeed, if $\bg$ is a sub-Gaussian vector in $\bbR^n$ with sub-gaussian parameter $\sigma$, then
  \begin{align*}
    \bbE\left[\|\bg\|\right] &\le \sum_{i=1}^n \bbE[|e_i^\top\bg|] = \sum_{i=1}^n \bbE[\max\{e_i^\top\bg,-e_i^\top\bg\}]
    \le n \sigma\sqrt{2 \log(2)} \le 2n\sigma.                              
  \end{align*}
  Here $e_i$ are the standard basis vectors of $\bbR^n$. 
  The third inequality is based on a standard bounding method for sub-Gaussian variables. See exercise 2.21 of \cite{wainwright2019high}. Additionally, we will work out these details for bounding a different maximum of sub-Gaussian variables.

   Thus, by integrating (\ref{eq:itoCS}), taking expectations, and noting that $t-k\le 1$, we can conclude that
  \begin{equation*}
    \bbE\left[\|\bx_t^C-\bx_k^C\|^2 \right] \le
    2\eta \left( \frac{n}{\beta} + uD + 2n\sigma \right)
  \end{equation*}
  Thus,  an elementary Cauchy-Schwarz bound gives that 
  \begin{align}
    \label{eq:tailNorm}
    \bbE\left[\|\bx_t^C-\bx_k^C\| \right]\le \sqrt{\bbE\left[
 \|\bx_t^C-\bx_k^C\|^2
    \right]} \le \sqrt{2\eta \left(Du + 2n\sigma + \frac{n}{\beta} \right)}.
  \end{align}

  The rest of the proof bounds $\bbE[\|\bx_k^C-\bx_k^D\|]$. When $k=0$, this term is $0$, so we focus on the $k\ge 1$ case. 
  Recall that $\bx_t^C$ solves the Skorokhod problem for  $\by_t^C$
  and $\bx_t^D$ solves the Skorkhod problem for $\cD(\by^C)_t =
  \by_{\floor*{t}}^C$. Let $\bvarphi_t^D = -\int_0^t \bv_t^D
  d\bmu^D(t)$ be the unique projection process such that $\bx_t^D =
  \by_{\floor*{t}}^C + \bvarphi_t^D$. 
  Then, Lemma 2.2 of \cite{tanaka1979stochastic}
  implies that
  \begin{align*}
    \MoveEqLeft
    \|\bx_k^C-\bx_k^D\|^2 \\
    &\le \|\by_k^C-\by_{k}^C\|^2 
    +2\int_0^k(\by_k^C-\by_{k}^C-\by_s^C+\by_{\floor*{s}}^C)^\top(\bv_s^Dd\bmu^D(s)-\bv_s^C
      d\bmu^C(s)) \\
    &= 2\int_0^k(\by_{\floor*{s}}^C-\by_s^C)^\top (\bv_s^Dd\bmu^D(s)-\bv_s
      d\bmu(s))
  \end{align*}

  Note that for any integer, $i$, $\by_{\floor*{s}}^C$ is constant for  $s\in (i,i+1)$. It follows that the measure, $\bmu^D$ is supported on the integers. However, the integrand is zero on the integers, so we arrive at the simplified bound:
\begin{equation*}
  \|\bx_k^C-\bx_k^D\|^2 \le 2 \int_0^k(\by_s^C-\by_{\floor*{s}}^C)^\top \bv_s^C d\bmu^C(s).
\end{equation*}

Now, the elementary inequality $x^\top y \le \gamma(x|\cK)\delta^*(x|\cK)$ followed by H\"older's inequality gives:
\begin{align*}
  \|\bx_k^C-\bx_k^D\|^2
  &\le 2\int_0^k \gamma(\by_s^C-\by_{\floor*{s}}^C|\cK) \delta^\star(\bv_s|\cK) d\bmu(s)
  \\
  &\le 2\left(\sup_{s\in[0,k]} \gamma(\by_s^C-\by_{\floor*{s}}^C|\cK) \right)
    \int_0^k \delta^\star(\bv_s|\cK)d\bmu(s)
\end{align*}

Taking square-roots, then followed by expectations, and then employing the Cauchy-Schwarz inequality gives:
\begin{align}
  \nonumber
  \bbE\left[\|\bx_k^C-\bx_k^D\| \right]
  &\le
    \sqrt{2}
    \bbE\left[
    \sqrt{\sup_{s\in[0,k]} \gamma(\by_s^C-\by_{\floor*{s}}^C|\cK)}\sqrt{\int_0^k \delta^\star(\bv_s|\cK)d\bmu(s)}
    \right] \\
  \label{eq:CDcauchySchwarz}
  &\le \sqrt{2 \bbE\left[ \sup_{s\in[0,k]} \gamma(\by_s^C-\by_{\floor*{s}}^C|\cK)\right]
    \bbE\left[ \int_0^k \delta^\star(\bv_s|\cK)d\bmu(s)\right]}.
\end{align}

So, now it suffices to bound both terms on the right of (\ref{eq:CDcauchySchwarz}).

We first bound the $\gamma$ term. The methodology deviates from that of \cite{bubeck2018sampling}, as this term is now a bit more complicated. 
First note that $\gamma(x|\cK)\le r^{-1}\|x\|$ because $\cK$ contains a ball of radius $r$ around the origin. Thus, plugging in the defition for $\by_t^C$ and using the triangle inequality gives:
\begin{align}
  \nonumber
  \gamma(\by_s^C-\by_{\floor*{s}}^C|\cK)
  &\le
    r^{-1} \|\by_s^C-\by_{\floor*{s}}^C\| \\
  \nonumber
  &\le r^{-1} \eta \int_{\floor*{s}}^s \| \nabla_x f(\bx_\tau^C,\bz_{\floor*{s}}) \|d\tau + r^{-1}\sqrt{\frac{2\eta}{\beta}} \|\bw_{s}-\bw_{\floor*{s}}\| 
\end{align}

For compact notation, set $i=\floor*{s}$ and $\bg_\tau = \nabla_x f(\bx_{\tau}^C,\bz_i)-\nabla_x \bar{f}(\bx_{\tau}^C)$.
To bound the integral term, note that
\begin{align}
  \nonumber
  \| \nabla_x f(\bx_\tau^C,\bz_{\floor*{s}}) \| &= \| \nabla_x \bar{f}(\bx_{\tau}^C) + (\bg_{\tau}-\bg_i) + \bg_i\| \\
  \label{eq:tripleSplit}
  &\le u + \ell D + \|\bg_i\|.
\end{align}
The inequality arises because of the bound on $\|\nabla_x \bar f\|$ and the Lipschitz property of $\nabla_x f$.

It follows that
\begin{equation}
  \label{eq:gaugeTriangle}
  \gamma(\by_s^C-\by_{\floor*{s}}^C|\cK) \le
  \frac{\eta (u+\ell D)}{2r} + \frac{\eta}{2r}\|\bg_i\| + r^{-1}\sqrt{\frac{2\eta}{\beta}}\|\bw_{s}-\bw_{\floor*{s}}\|.
\end{equation}

Thus, to bound the first term on the right of (\ref{eq:CDcauchySchwarz}), it suffices to bound $\max_{i=0,\ldots,k-1} \|\bg_i\|$ and $\sup_{s\in[0,k]} \|\bw_s-\bw_{\floor*{s}}\|$. The $\|\bg_i\|$ terms can be bounded using a modification of a standard sub-Gaussian bounding method from exercise 2.21 of \cite{wainwright2019high}. We show it explicitly, as the method will be generalized when bounding the $\|\bw_s-\bw_{\floor*{s}}\|$ terms.

Let $e_j$ be the standard basis vectors of $\bbR^n$. Then the following bound follows from the triangle inequality:
\begin{align*}
  \|\bg_i\| &\le \sum_{j=1}^n |e_j^\top \bg_i| \\
  &= \sum_{j=1}^n \max_{\varepsilon \in \{-1,1\}} \varepsilon e_j^\top \bg_i.
\end{align*}

Then for any  $\lambda >0$ we have that
\begin{align*}
  \max_{i=0,\ldots,k-1} \|\bg_i\| &\le \max_{i=0,\ldots,k-1} \sum_{j=1}^n \max_{\varepsilon \in \{-1,1\}} \varepsilon e_j^\top \bg_i \\
                                  &\le \sum_{j=1}^n \max_{i\in \{0,\ldots,k-1\},\varepsilon \in \{-1,1\}} \varepsilon e_j^\top \bg_i \\
                                  &\le \sum_{j=1}^n \lambda^{-1} \log\left(
                                    \sum_{i=0}^{k-1}\sum_{\varepsilon\in\{-1,1\}} \exp\left(\lambda \varepsilon e_j^\top \bg_i  \right)
                                    \right).
\end{align*}

Taking expectations and using Jensen's inequality, followed by the sub-Gaussian property of $\bg_i$ gives
\begin{align}
  \nonumber
  \bbE\left[
  \max_{i=0,\ldots,k-1} \|\bg_i\|
  \right] &\le \sum_{j=1}^n \lambda^{-1} \log\left(
 \sum_{i=0}^{k-1}\sum_{\varepsilon\in\{-1,1\}} \bbE\left[\exp\left(\lambda \varepsilon e_j^\top \bg_i  \right) \right]
            \right) \\
          &\le \frac{n}{\lambda}\log\left(2k e^{\lambda^2 \sigma^2/2} \right) \\
  &= \frac{n\log(2k)}{\lambda} + \frac{n\lambda \sigma^2}{2}.
\end{align}
Optimizing over $\lambda$ gives:
\begin{equation}
  \label{eq:gradMax}
  \bbE\left[
  \max_{i=0,\ldots,k-1} \|\bg_i\|
  \right] \le n\sigma \sqrt{2 \log(2k)}.
\end{equation}

Now we will bound $\bbE\left[\sup_{s\in[0,k]} \|\bw_s-\bw_{\floor*{s}}\|\right]$ using an extension of the argument just used. Note that

\begin{align*}
\bbE\left[
\sup_{s\in [0,k]} \|\bw_{s}-\bw_{\floor*{s}}\|
  \right] &=
  \bbE\left[
\max_{i=0,\ldots,k-1} \sup_{s\in[i,i+1]} \|\bw_{s}-\bw_i\|
            \right]
\end{align*}

Then the triangle inequality implies that 
\begin{equation*}
  \|\bw_{s}-\bw_i\| \le \sum_{j=1}^n |e_j^\top  (\bw_{s}-\bw_i)|
  = \sum_{j=1}^n \max_{\varepsilon\in\{-1,1\}}\varepsilon e_j^\top (\bw_s-\bw_i) .
\end{equation*}
It follows that for all $\lambda >0$, we get:
\begin{align*}
  \max_{i=0,\ldots,k-1} \sup_{s\in[i,i+1]} \|\bw_{s}-\bw_i\| &\le \sum_{j=1}^n \max_{i\in \{0,\ldots,k\},\varepsilon\in \{-1,1\}}  \sup_{s\in[i,i+1]} \varepsilon e_j^\top (\bw_s-\bw_i) \\                                                         &\le                                                                                                                                       \sum_{j=1}^n\lambda^{-1} \log\left(                                                                                                 \sum_{i=0}^{k-1}\sum_{\varepsilon\in\{-1,1\}}\sup_{s\in[i,i+1]}                                                              e^{\lambda \varepsilon e_j^\top (\bw_s-\bw_i)}                                                                                                                    \right) 
\end{align*}

So, taking expectations and using Jensen's inequality gives:
\begin{align}
  \nonumber
  \MoveEqLeft
  \bbE\left[ \max_{i=0,\ldots,k-1} \sup_{s\in[i,i+1]} \|\bw_{s}-\bw_i\| \right] \\
  \nonumber
  &\le 
    \sum_{j=1}^n\lambda^{-1} \bbE\left[\log\left(  \sum_{i=0}^{k-1}\sum_{\varepsilon\in\{-1,1\}}\sup_{s\in[i,i+1]}                                                              e^{\lambda \varepsilon e_j^\top (\bw_s-\bw_i)}
    \right) \right] \\
  \label{eq:logsumexp}
  &\le 
    \sum_{j=1}^n\lambda^{-1} \log\left(  \sum_{i=0}^{k-1}\sum_{\varepsilon\in\{-1,1\}}
    \bbE\left[    \sup_{s\in[i,i+1]}                                                              e^{\lambda \varepsilon e_j^\top (\bw_s-\bw_i)}
    \right]
    \right) 
\end{align}

Now we bound the expectation of each term on the right of (\ref{eq:logsumexp}).
For simple notation, let $\alpha = \varepsilon e_j$ correspond to one of the terms in the sum. Note that $\alpha$ is a unit vector and that $e^{\lambda \alpha^\top (\bw_s-\bw_i)}$ is convex with respect to $\bw_s$. Now since, $\bw_s$ is martingale, it follows that $e^{\lambda \alpha^\top(\bw_s-\bw_i)}$ is a submartingale for $s\in[i,i+1]$. So, a Cauchy-Schwarz bound followed by Doob's maximal inequality, and then direct computation gives: 
\begin{align*}
  \bbE\left[\sup_{s\in[i,i+1]} e^{\lambda \alpha^\top(\bw_s-\bw_i)} \right]
  &\le 
  \sqrt{\bbE\left[\sup_{s\in[i,i+1]} e^{2\lambda \alpha^\top(\bw_s-\bw_i)} \right]} \\
  &\le 2 \sqrt{\bbE\left[ e^{2\lambda \alpha^\top(\bw_{i+1}-\bw_i)} \right]} \\
  &=2e^{\lambda^2}
\end{align*}
Combining this result with (\ref{eq:logsumexp}) shows that 
\begin{align*}
\bbE\left[
\sup_{s\in [0,k]} \|\bw_{s}-\bw_{\floor*{s}}\| 
  \right]
  &\le \frac{n}{\lambda}\log(4ke^{\lambda^2}) = \frac{n\log(4k)}{\lambda} + n\lambda
\end{align*}
for all $\lambda >0$.

Optimizing over $\lambda$ shows that
\begin{align*}
\bbE\left[
\sup_{s\in [0,k]} \|\bw_{s}-\bw_{\floor*{s}}\| 
  \right]
\le 2n\sqrt{\log(4k)}
\end{align*}
Combining  this result with (\ref{eq:gaugeTriangle}) and (\ref{eq:gradMax}) shows that
\begin{align}
  \nonumber
  \bbE\left[\sup_{s\in[0,k]}\gamma(\by_s^C-\by_{\floor*{s}}^C|\cK)\right]
  & \le \frac{\eta (u+\ell D + n\sigma \sqrt{2\log(2k)}) }{2r} + \frac{2n}{r} \sqrt{\frac{2\eta \log(4k)}{\beta}} \\
  \label{eq:gaugeBoundBasic}
  &\le \sqrt{\eta \log (4k)}\left( \frac{u+\ell D}{2r} +\frac{n \sigma}{\sqrt{2}r} +\frac{2n\sqrt{2}}{r\sqrt{\beta}}\right). 
\end{align}
The second inequality used the assumption that $\eta\le 1$ so that $\eta\le \sqrt{\eta}$, and the fact that $\log(4k)\ge 1$ for $k\ge 1$. 

  Now we bound the second term on the right of (\ref{eq:CDcauchySchwarz}).
Note that $\bx_t^C$ is a continuous semimartingale and the process
$\int_0^t\bv_s d\bmu(s)$ has bounded variation. Thus, from It\^o's
formula, \cite{kallenberg2002foundations}, we have that
\begin{equation}
  \label{eq:squareDifferential}
  d\|\bx_t^C\|^2 = 2(\bx_t^C)^\top \left(-\eta \nabla_x f(\bx_t^C,\bz_{\floor*{t}})
    +\sqrt{\frac{2\eta}{\beta}}d\bw_t -\bv_t d\bmu(t)
    \right)+ \frac{2\eta n}{\beta} dt
  \end{equation}

Reasoning as in (\ref{eq:gradTriangle}) shows that 
\begin{equation*}
  \bbE\left[|(\bx_t^C)^\top \nabla f_x(\bx_t^C,\bz_{\floor*{t}})|
  \right] \le Du + D\bbE\left[\|\bg_t\|\right]\le Du + 2Dn\sigma.
\end{equation*}
  
  By construction, $(\bx_t^C)^\top \bv_t = \sup \{x^\top\bv_t | x\in\cK
  \}=\delta^\star(\bv_t|\cK)$.
  Thus, re-arranging, integrating, and
    taking expectations gives the bound:
    \begin{align}
      \nonumber
      \MoveEqLeft
      \bbE\left[
        \int_0^{t} \delta^\star(\bv_s|\cK)d\bmu(s)
      \right] \\
      \nonumber
      &
      = \frac{\eta n t}{\beta} -\eta\bbE\left[\int_0^{t} (\bx_s^C)^\top
        \nabla_x f(\bx_s^C,\bz_{\floor*{s}})ds
      \right] + \frac{1}{2}\bbE\left[\|\bx_0^C\|^2-\|\bx_{t}^C\|^2 \right] \\
      \label{eq:supportBound}
      &\le \eta t \left(
        \frac{n}{\beta} + Du + 2Dn\sigma
        \right)+ \frac{D^2}{2} 
    \end{align}

    Combining (\ref{eq:CDcauchySchwarz}), (\ref{eq:gaugeBoundBasic}), and (\ref{eq:supportBound}) shows that 
    \begin{align*}
\bbE\left[
      \|
      \bx_k^C-\bx_k^D\|
      \right] 
      &\le \sqrt{2}\sqrt{\sqrt{\eta \log (4k)}\left( \frac{u+\ell D}{2r} +\frac{n \sigma}{\sqrt{2}r} +\frac{2n\sqrt{2}}{r\sqrt{\beta}}\right)} \\
      & \cdot \sqrt{\eta k \left(
        \frac{n}{\beta} + Du + 2Dn\sigma
        \right)+ \frac{D^2}{2} }
    \end{align*}

Combining this result with (\ref{eq:tailSplit}) and (\ref{eq:tailNorm}) finishes the proof.
\hfill$\blacksquare$

\paragraph{Proof of Lemma~\ref{lem:linUpper}}
  Recall that $\bx_t^A$ and $\bx_t^D$ are discretized processes: $\bx_t^A=\bx_{\floor*{t}}^A$ and $\bx_t^D=\bx_{\floor*{t}}^D$.
  Furthermore, if we set $\by_t^A$ as
  $$
  \by_t^A = \bx_0^A + \eta \int_0^t \nabla_x f(\bx_{\floor*{s}}^A,\bz_{\floor*{s}})ds + \sqrt{\frac{2\eta}{\beta}}\bw_t,
  $$
  then we have $\bx^A = \cS(\cD(\by^A))$ and $\bx^D=\cS(\cD(\by^C))$, where $\cD$ is the discretization operator and $\cS$ is the Skorokhod solution operator. In particular
  $$
  \bx_{k+1}^A = \Pi_{\cK}(\bx_k^A+\by_{k+1}^A-\by_k^A).
  $$

  Define a difference process, $\brho_t$, by:
  $$\brho_t = (\bx^A_t + \by_t^A-\by_{\floor*{t}}^A) -
  (\bx^D_t + \by_t^C-\by_{\floor*{t}}^C)
  $$
  Note that for integers $k$, $\brho_k=\bx_k^A-\bx_k^D$. While $\brho_t$ can jump at the integers,  non-expansiveness of convex projections implies that
  \begin{equation}
  \label{eq:nonExpansive}
  \|\brho_k\|=\|\bx_k^A-\bx_k^D\|\le \lim_{t\uparrow k} \|\brho_t\|
  \end{equation}

  Let $k\ge 0$ be an integer. 
  For $t\in[k,k+1)$ we have that
  \begin{equation*}
    d\brho_t = d(\by_t^A-\by_t^C) = \eta (\nabla_x f(\bx_t^C,\bz_{\floor*{t}})-\nabla_x f(\bx_t^A,\bz_{\floor*{t}})) dt
  \end{equation*}

  It follows that $\brho_t$ is a continuous bounded variation process on the interval $[k,k+1)$. When $\brho_t\ne 0$, we can bound the growth of $\|\brho_t\|$ using the chain rule, followed by the Cauchy-Schwarz inequality, the Lipshitz property of $\nabla_x f$, and the triangle inequality:
  \begin{align}
    \nonumber
    d\|\brho_t\| &= \left(\frac{\brho_t}{\|\brho_t\|}\right)^\top \eta (\nabla_x f(\bx_t^C,\bz_{\floor*{t}})-\nabla_x f(\bx_t^A,\bz_{\floor*{t}})) dt \\
                 &\le \eta \ell \|\bx_t^C-\bx_t^A\|dt \\
    \label{eq:diffTriangle}
    &\le \eta \ell (\|\bx_t^C-\bx_t^D\|+ \|\bx_t^D-\bx_t^A\|)dt.
  \end{align}

  While we have not characterized the behavior when $\brho_t=0$, the Lemma~\ref{lem:variationAtZero} from Appendix~\ref{app:integration} can be used to show that this behavior does not cause problems. Specifically for $t\in [k,k+1)$

  \begin{align*}
    \|\brho_t\|&=\|\brho_k\| + \int_k^td\|\brho_s\| \\
    &\overset{Lem.~\ref{lem:variationAtZero}}{=}\|\brho_k\| + \lim_{\epsilon \downarrow 0}\int_k^t \indic(\|\brho_s\| \ge \epsilon) d\|\brho_s\|\\
               &\overset{~(\ref{eq:diffTriangle})}{\le}\|\brho_k\| + \lim_{\epsilon \downarrow 0}\int_k^t  \indic(\|\brho_s\| \ge \epsilon) \eta \ell (\|\bx_s^C-\bx_s^D\|+ \|\bx_s^D-\bx_s^A\|)ds \\
               &\le (1+\eta \ell) \|\brho_k\| + \eta \ell\int_k^{t} \|\bx_s^C-\bx_s^D\|ds
  \end{align*}

  The final inequality used the fact that $\brho_k = \bx_s^A-\bx_s^D$ for all $s\in[k,k+1)$.

  Now using (\ref{eq:nonExpansive}) we see that
  \begin{equation*}
    \|\brho_{k+1}\|\le (1+\eta \ell) \|\brho_k\| +\eta \ell\int_k^{k+1} \|\bx_s^C-\bx_s^D\|ds 
  \end{equation*}
  Then using the assumption that $\brho_0=\bx_0^A-\bx_0^D=0$, we have that
  \begin{equation*}
    \|\brho_k\|\le \sum_{i=0}^{k-1} \eta \ell (1+\eta\ell)^{k-i-1} \int_i^{i+1}\|\bx_s^C-\bx_s^D\|ds 
  \end{equation*}

  Taking expectations and using Lemma~\ref{lem:tanakaMean} gives that
  \begin{align*}
\bbE[ \|\brho_k\|] &\le \eta \ell \sum_{i=1}^{k-1} (1+\eta\ell)^{k-i-1} \int_i^{i+1} \left( \eta \log(4\max\{1,s\})\right)^{1/4} 
  \left( c_{\ref{tanakaRt}} \sqrt{\eta s} + c_{\ref{tanakaConst}} \right)ds
    \\
    &\le \eta \ell  \left( \eta \log(4\max\{1,k\})\right)^{1/4} 
      \left( c_{\ref{tanakaRt}} \sqrt{\eta k} + c_{\ref{tanakaConst}} \right) \sum_{i=1}^{k-1} (1+\eta\ell)^{k-i-1} \\
    &\le  \left( \eta \log(4\max\{1,k\})\right)^{1/4} 
      \left( c_{\ref{tanakaRt}} \sqrt{\eta k} + c_{\ref{tanakaConst}} \right) ((1+\eta\ell)^k -1)
  \end{align*}
  The result now follows because $\bx_t^A$ and $\bx_t^D$ are constant for $t\in [k,k+1)$ and the bound above is monotinically increasing in $k$.
\hfill$\blacksquare$

\paragraph{Proof of Lemma~\ref{lem:dudley}}
  Let $-\int_0^t \bv_s^B d\bmu^B(s)$ be the unique finite-variation process that enforces that $\bx_t^B\in \cK$ in the Skorokhod solution.
  Lemma 2.2 of \cite{tanaka1979stochastic} implies that
  \begin{multline}
    \nonumber
    \|\bx_t^B-\bx_t^M\|^2 \le \|\by_t^B-\by_t^M\|^2 \\
    + 2\int_0^t\left(
      \by_t^B-\by_t^M-\by_s^B+\by_s^M
      \right)^\top \left(\bv_s^Md\bmu^M(s)-
        \bv_s^Bd\bmu^B(s)
        \right)
      \end{multline}

 Thus, taking square roots and using the triangle inequality gives
  \begin{multline}
    \label{eq:averagineSkorkhodBound}
    \|\bx_t^B-\bx_t^M\| \le \|\by_t^B-\by_t^M\| \\
    + \sqrt{2\left|\int_0^t\left(
      \by_t^B-\by_t^M-\by_s^B+\by_s^M
    \right)^\top \bv_s^Md\bmu^M(s)\right|}\\
+ \sqrt{2\left|\int_0^t\left(
      \by_t^B-\by_t^M-\by_s^B+\by_s^M
      \right)^\top \bv_s^Bd\bmu^B(s)\right|}
      \end{multline}
 
      Now we analyze the various terms of this equation.

      First we will bound $\bbE[\|\by_t^B-\by_t^M\|]\le \sqrt{ \bbE[\|\by_t^B-\by_t^M\|^2 ]}$.
      
      Let $\cF_{\infty}$ be the $\sigma$-algebra generated by the Brownian motion. In the following discussion, we will assume that the realization of the Brownian motion is fixed and examine the effects of $\bz_t$. Note that the initial condition assumption and the definition of $\by_t^B$ from (\ref{eq:Ybetween}) imply that $\by_t^M = \bbE[\by_t^B | \cF_\infty]$. In other words, $\by_t^B-\by_t^M$ is a zero-mean function of the random variables $\bz_0,\bz_1,\ldots$.

      To bound $\bbE\left[\|\by_t^B-\by_t^M\|^2 | \cF_\infty\right]$, it suffices to bound the individual coordinates. Each coordinate can be represented as $e_j^\top(\by_t^B-\by_t^M)$, where $e_j$ is a corresponding unit basis vector. Thus, it suffices to bound $\bbE\left[ \left(\alpha^\top(\by_t^B-\by_t^M)\right)^2\vert \cF_\infty \right]$ for an arbitrary unit vector, $\alpha$.

With the realization of the Brownian motion fixed, $\bv_t:=\alpha^\top(\by_t^B-\by_t^M)$ can be decomposed as a sum of independent, sub-Gaussian random variables:
      \begin{align*}
        \nonumber
        \bv_t=\alpha^\top(\by_t^B-\by_t^M) &= \eta \sum_{i=0}^{\floor*{t}-1} \int_i^{i+1} \alpha^\top \left(\nabla_x \bar f(\bx_s^M)-\nabla_x f(\bx_s^M,\bz_i) \right)ds \\
        &+\eta \int_{\floor*{t}}^t \alpha^\top \left(\nabla_x \bar f(\bx_s^M)-\nabla_x f(\bx_s^M,\bz_i) \right)ds
        \\
        &=: \sum_{i=0}^{\floor*{t}}\brho_i
      \end{align*}

      Recall that $\nabla_x \bar f(\bx_s^M)-\nabla_x f(\bx_s^M,\bz_i)$ is sub-Gaussian for all $\bx_s^M$. In particular, for $i<\floor*{t}$, we have for all $\lambda\in\bbR$,
      \begin{align*}
        \MoveEqLeft
        \bbE[
        \exp(\lambda \brho_i) |\cF_{\infty}
        ] \\
        &= \bbE\left[ \exp\left(\int_{i}^{i+1} \lambda\eta \alpha^\top \left(\nabla_x \bar f(\bx_s^M)-\nabla_x f(\bx_s^M,\bz_i) \right)ds  \right)
            \middle| \cF_{\infty}
            \right] \\
        &\overset{Jensen+Fubini}{=}
          \int_i^{i+1}\bbE\left[ \exp\left( \lambda\eta\alpha^\top \left(\nabla_x \bar f(\bx_s^M)-\nabla_x f(\bx_s^M,\bz_i) \right)  \right)
            \middle| \cF_{\infty}
          \right] ds \\
        &\overset{sub-Gaussian}{\le} \int_i^{i+1} \exp\left(\frac{1}{2}\lambda^2\eta^2\sigma^2  \right) ds \\
        &= \exp\left(\frac{1}{2}\lambda^2\eta^2\sigma^2  \right) .
      \end{align*}

Now consider the case that $i=\floor*{t}$. When $t=\floor*{t}=i$, we have that $\brho_i=0$ and so $\bbE[\exp(\lambda\brho_i)|\cF_{\infty}]=1$. When $t>\floor*{t}$, a similar argument as above gives:
       \begin{align*}
        \MoveEqLeft[0]
        \bbE[
        \exp(\lambda \brho_i) |\cF_{\infty}
        ] \\
         &= \bbE\left[ \exp\left(
\frac{1}{t-\floor*{t}}
           \int_{\floor*{t}}^{t} \lambda\eta(t-\floor*{t})\alpha^\top \left(\nabla_x \bar f(\bx_s^M)-\nabla_x f(\bx_s^M,\bz_i) \right)ds  \right)
            \middle| \cF_{\infty}
            \right] \\
        &\overset{J+F}{=} \frac{1}{t-\floor*{t}}
          \int_{\floor*{t}}^{t}\bbE\left[ \exp\left( \lambda\eta(t-\floor*{t})\alpha^\top \left(\nabla_x \bar f(\bx_s^M)-\nabla_x f(\bx_s^M,\bz_i) \right)  \right)
            \middle| \cF_{\infty}
          \right] ds \\
         &\overset{sub-Gaussian}{\le}
\frac{1}{t-\floor*{t}}
          \int_{\floor*{t}}^{t}
           \exp\left(\frac{1}{2}\lambda^2\eta^2(t-\floor*{t})^2\sigma^2  \right) ds \\
        &\le \exp\left(\frac{1}{2}\lambda^2\sigma^2\eta^2 (t-\floor*{t}) \right) .
      \end{align*}
     The final inequality used the fact that $(t-\floor*{t})^2 \le t-\floor*{t}$.

      Now using the fact that $\brho_i$ are independent, conditioned on $\cF_{\infty}$, we have that
      \begin{align*}
        \bbE\left[
        \exp\left(
        \lambda
        \bv_t
        \right)
        |\cF_\infty
        \right]
        &= \prod_{i=0}^{\floor*{t}}
          \bbE\left[
          e^{\lambda \brho_i}|\cF_\infty
          \right] \\
        &\le e^{(\lambda\eta \sigma)^2 (t-\floor*{t}))/2} \prod_{i=0}^{\floor*{t}-1} e^{(\lambda\eta \sigma)^2/2}
        \\
        &\le e^{\lambda^2 \eta^2 \sigma^2 t /2}.
      \end{align*}

      Thus we have shown that $\bv_t$ is sub-Gaussian with parameter $\hat\sigma^2=\eta^2 \sigma^2t$. Then a standard Chernoff bound argument shows that $\bbP(|\bv_t|^2 > \epsilon|\cF_{\infty})\le 2e^{-\epsilon /(2\hat\sigma^2)}$. Then we can bound the variance by:
      \begin{align}
        \nonumber
        \bbE\left[
        \left(
        \alpha^\top(\by_t^B-\by_t^M)
        \right)^2
        \vert \cF_\infty
        \right] &= \int_0^\infty \bbP(|\bv_t|^2 > \epsilon|\cF_{\infty}) d\epsilon \\
        \nonumber
                &\le 2 \int_0^\infty e^{-\epsilon/(2\hat\sigma^2)} d\epsilon \\
        \nonumber
                &= 4\hat\sigma^2 \\
        \label{eq:coordinateVarBound}
        &= 4\eta^2 \sigma^2 t.
      \end{align}

      Applying (\ref{eq:coordinateVarBound}) to $\alpha=e_j$ for all of the standard basis vectors, then summing and using the tower property gives:
      \begin{equation*}
        \bbE\left[
          \|\by_t^B-\by_t^M\|^2
        \right] \le 4n\eta^2 \sigma^2 t.
      \end{equation*}
      
Taking square roots gives
\begin{equation}
\label{eq:endBound}
 \bbE\left[
          \|\by_t^B-\by_t^M\|
        \right] \le \sqrt{\bbE\left[
          \|\by_t^B-\by_t^M\|^2
        \right]} \le 2\sigma \eta \sqrt{nt}.
\end{equation}

      Bounding the integral terms from (\ref{eq:averagineSkorkhodBound}) is more complex. First we consider the integral with rerspect to $\bmu^M$. The integral with respect to $\bmu^B$ is similar. As in the proof of Lemma~\ref{lem:tanakaMean} we will use a H\"older inequality bound, followed by a Cauchy-Schwarz bound:
      \begin{align}
        \nonumber
        \MoveEqLeft
        \bbE\left[
        \sqrt{
        \left|
        \int_0^t\left(
      \by_t^B-\by_t^M-\by_s^B+\by_s^M
        \right)^\top \bv_s^Md\bmu^M(s)
        \right|
        }
        \right] \\
        \nonumber
        &\le \bbE\left[
          \sqrt{
          \int_0^t
          \gamma(\by_t^B-\by_t^M-\by_s^B+\by_s^M|\cK)
          \delta^\star(\bv_s^M|\cK)d\bmu^M(s)}
          \right] \\
        &\le \bbE\left[
          \sqrt{
          \sup_{s\in [0,t]} \gamma(\by_t^B-\by_t^M-\by_s^B+\by_s^M|\cK)}
          \sqrt{
          \int_0^t \delta^\star(\bv_s^M|\cK)d\bmu^M(s)
          }
          \right] \\
        \label{eq:averageHolder}
        &\le \sqrt{\bbE\left[
          \sup_{s\in [0,t]} \gamma(\by_t^B-\by_t^M-\by_s^B+\by_s^M|\cK)\right]}
          \sqrt{\bbE\left[
          \int_0^t \delta^\star(\bv_s^M|\cK)d\bmu^M(s)
          \right]
          }
      \end{align}

      The integral bound follows from (\ref{eq:supportBound}) applied to $\bar f$ in place of $f$:
      \begin{equation}
        \label{eq:fullSupport}
\bbE\left[
          \int_0^t \delta^\star(\bv_s^M|\cK)d\bmu^M(s)
          \right] \le
\eta t \left(
        \frac{n}{\beta} + Du + 2Dn\sigma
        \right)+ \frac{D^2}{2} 
      \end{equation}

      Bounding the supremum will take more work. The eventual plan is to bound the individual components using the Dudley entropy integral. 
      To this end, we first note that for any vector in $x\in\bbR^n$, we have that 
      \begin{align}
        \label{eq:gaugeSplit}
        \gamma(x|\cK) \le r^{-1} \|x\| \le r^{-1} \sum_{i=1}^n |x_i|.
      \end{align}
      Thus, it suffices to bound $\bbE[\sup_{s\in[0,t]} |\bv_t-\bv_s| \vert \cF_{\infty}]$,
      where
      \begin{equation*}
        \bv_s = \alpha^\top \left(
\by_s^B-\by_s^M
          \right).
        \end{equation*}
        and $\alpha$ is an arbitrary unit vector.

 Also note that 
 \begin{equation*}
   \label{eq:supRelax}
   \sup_{s\in [0,t]} |\bv_t-\bv_s|
   \le \sup_{s,\hat s \in [0,t]}(\bv_s-\bv_{\hat s}).
 \end{equation*}
 The expectation of the expression on the right will now be bounded via the Dudley entropy integral. To derive the bound, we must show that $\bv_s$ has sub-Gaussian increments. A mild extension of the argument that $\bv_t$ is sub-Guassian will suffice.

 Without loss of generality, assume that $\hat s \ge s$. Then
 \begin{align*}
   \bv_{\hat s} - \bv_{s}
   &= \eta \int_s^{\hat s} \alpha^\top \left(
     \nabla_x \bar f(\bx_{\tau}^M)-\nabla_x \bar f(\bx_\tau^M,\bz_{\floor*{\tau}})
     \right)d\tau \\
   &= \eta \int_{s}^{\ceil*{s}} \alpha^\top \left(
     \nabla_x \bar f(\bx_{\tau}^M)-\nabla_x \bar f(\bx_\tau^M,\bz_{\floor*{\tau}})
     \right)d\tau + \\
   &
   \eta \sum_{i=\ceil*{s}}^{\floor*{\hat s}-1}
   \int_{i}^{i+1}\alpha^\top \left(
     \nabla_x \bar f(\bx_{\tau}^M)-\nabla_x \bar f(\bx_\tau^M,\bz_{\floor*{\tau}})
     \right)d\tau + \\
   &= \int_{\floor*{s}}^{s}\alpha^\top \left(
     \nabla_x \bar f(\bx_{\tau}^M)-\nabla_x \bar f(\bx_\tau^M,\bz_{\floor*{\tau}})
     \right)d\tau \\
   &=: \sum_{i=\ceil*{s}-1}^{\floor*{\hat s}}\brho_i
 \end{align*} 
 Then, similar to the case above,  $\brho_i$ are independent sub-Gaussian random variables with the following bounds for all $\lambda\in\bbR$: 
 \begin{equation*}
   \bbE[e^{\lambda \brho_i}|\cF_{\infty}]
   \le \begin{cases}
     e^{(\eta \sigma \lambda)^2(\ceil*{s}-s)/2} & \textrm{ if } i=\ceil*{s}-1 \\
     e^{(\eta \sigma \lambda^2)/2} & \textrm{ if } i=\ceil*{s},\ldots,\floor*{\hat s}-1 \\
     e^{(\eta\sigma\lambda)^2(\hat s-\floor*{s})/2} &\textrm{ if } i=\floor*{\hat s}
   \end{cases}
 \end{equation*}
 Then independence implies that for all $\lambda\in\bbR$, the following bound holds:
 \begin{align*}
   \bbE[e^{\lambda (\bv_s-\bv_{\hat s})} |\cF_{\infty}]
   &
    \le
   \exp\left(\frac{\lambda^2\eta^2 \sigma^2}{2}\left(
     (\floor*{s}-s)^2 + \sum_{i=\ceil*{s}}^{\floor*{\hat s}-1} 1
     +(\hat s - \floor*{\hat s})^2 \right)
     \right) \\
   &\le \exp\left(\frac{\lambda^2\eta^2 \sigma^2|s-\hat s|}{2} \right).
 \end{align*}
 It follows that $\bv_s$ is sub-Guassian with respect to the metric defined by
 $d(s,\hat s) = \eta \sigma \sqrt{|s-\hat s|}$. See Definition 5.16 of \cite{wainwright2019high}.

 Let $N([0,t],d,\epsilon)$ be the covering number of the interval $[0,t]$ via closed balls of  radius $\epsilon$ under the metric $d$, and similarly $N([0,t],|\cdot|,\epsilon)$ is the corresponding covering number with respect to the absolute value metric. A standard argument shows that $N([0,t],|\cdot|,\epsilon) = 1$ when $\epsilon \ge t/2$ and when $\epsilon \le t/2$, we have that
 \begin{equation*}
   N([0,t],|\cdot|,\epsilon)\le \frac{t}{2 \epsilon} + 1 \le t/\epsilon.
 \end{equation*}
 See Example 5.2 of \cite{wainwright2019high}.
 
 By the definition of $d$, a ball of radius $\epsilon$ in the $d$ metric corresponds to a ball of radius $\left(\frac{\epsilon}{\eta \sigma}\right)^2$ in the absolute value metric. And so, when $\epsilon \ge \eta \sigma \sqrt{\frac{t}{2}}$, we have that $N([0,t],d,\epsilon)=1$ and when $\epsilon \le  \eta \sigma \sqrt{\frac{t}{2}}$ the following bound holds:
 \begin{equation*}
N([0,t],d,\epsilon)\le \frac{t\eta^2\sigma^2}{\epsilon^2}
\end{equation*}
Thus, the Dudley entropy integral bound (see Theorem 5.22 of \cite{wainwright2019high}) implies that
\begin{align}
  \MoveEqLeft
  \nonumber
  \bbE[\sup_{s,\hat s \in[0,t]} (\bv_s-\bv_{\hat s})|\cF_{\infty}] \\
  \nonumber
  &
                                                                    \le 32 \int_0^{\eta \sigma \sqrt{\frac{t}{2}}} \sqrt{\log N([0,t],d,\epsilon)} d\epsilon \\
  \nonumber
                                                     &\le 32 \int_0^{\eta \sigma\sqrt{t}} \sqrt{\log\left(\frac{t\eta^2 \sigma^2}{\epsilon^2} \right)}d\epsilon \\
  \nonumber
                                                     &= 32\eta \sigma \sqrt{2t} \int_0^\infty x^{1/2} e^{-x} dx \quad \left(\textrm{using }2x=\log\left(\frac{t\eta^2 \sigma^2}{\epsilon^2} \right) \right) \\
  \label{eq:dudley}
  &= 16\eta \sigma \sqrt{2\pi t}
\end{align}

Thus, applying (\ref{eq:dudley}) for all of the standard basis vectors and plugging the bound into 
(\ref{eq:gaugeSplit}) and using the tower property gives 
\begin{equation}
  \label{eq:dudleyGauge}
  \bbE\left[
    \sup_{s\in [0,t]} \gamma(\by_t^B-\by_t^M-\by_s^B+\by_s^M|\cK)
  \right] \le \frac{16 n\eta \sigma \sqrt{2\pi t}}{r}
\end{equation}

Combining \eqref{eq:averageHolder}, \eqref{eq:fullSupport}, and \eqref{eq:dudleyGauge} shows that 
\begin{multline}
  \label{eq:finalAveHolder}
\bbE\left[\sqrt{
    \left|
        \int_0^t\left(
      \by_t^B-\by_t^M-\by_s^B+\by_s^M
        \right)^\top \bv_s^Md\bmu^M(s)
        \right|
        }
      \right] \le \\
      \sqrt{
      \frac{16 n\eta u \sqrt{2\pi t}}{r} \left(
 \eta t \left(
        \frac{n}{\beta} + Du  + 2Dn\sigma
        \right) + \frac{1}{2} D^2.
        \right)}
    \end{multline}

    An identical argument holds for the integral with respect to $\bv_s^Bd\bmu^M(s)$. So, multiplying the bound from (\ref{eq:finalAveHolder}) by $2\sqrt{2}=\sqrt{8}$ and adding it to (\ref{eq:endBound})  shows that
    \begin{align*}
      \MoveEqLeft
      W_1(\cL(\bx_t^B),\cL(\bx_t^M))
\\
      &
                                      \le
                                      \bbE\left[ \|\bx_t^B-\bx_t^M\|\right] \\
      &
      \le
2\sigma \eta \sqrt{nt}
       +   \sqrt{\frac{128 n\eta \sigma\sqrt{2\pi t}}{r} \left(
 \eta t \left(
        \frac{n}{\beta} + Du + 2Dn\sigma
        \right) + \frac{1}{2} D^2
        \right)}.
    \end{align*}
The result follows by factoring out the terms depending on $\eta$ and $t$.
\hfill$\blacksquare$

\paragraph*{Proof of Lemma~\ref{lem:gronwall}}
  Note that
  \begin{equation*}
    d(\bx_t^C-\bx_t^B) = \eta \left(
      \nabla_x f(\bx_t^M,\bz_{\floor*{t}})-\nabla_x f(\bx_t^C,\bz_{\floor*{t}})
    \right)dt
    +\bv_t^Bd\bmu^B(t)-\bv_t^Cd\bmu^C(t)
  \end{equation*}
  so that $\bx_t^C-\bx_t^B$ is a continuous bounded-variation process. Thus, whenever $\bx_t^C\ne \bx_t^B$, we have that:
  \begin{align*}
    d\|\bx_t^C-\bx_t^B\| 
    &= \left(\frac{\bx_t^C-\bx_t^B}{\|\bx_t^C-\bx_t^B\|}\right)^\top
      \eta 
      \left(
      \nabla_x f(\bx_t^M,\bz_{\floor*{t}})-\nabla_x f(\bx_t^C,\bz_{\floor*{t}})
      \right)dt \\
     &
     +\left(\frac{\bx_t^C-\bx_t^B}{\|\bx_t^C-\bx_t^B\|}\right)^\top \left(\bv_t^Bd\bmu^B(t)-\bv_t^Cd\bmu^C(t)
       \right) \\
    &\le \eta \ell \|\bx_t^M-\bx_t^C\| dt \\
    &\le \eta \ell (\|\bx_t^M-\bx_t^B\|+\|\bx_t^C-\bx_t^B\|)dt.
  \end{align*}
  The first inequality uses the definitions of $\bv_t^B$ and $\bv_t^C$ to imply that the corresponding terms are non-positive. It also simplifies the inner product with the gradients via the Lipschitz property and the Cauchy-Schwarz inequality. The second inequality uses the triangle inequality.

  Now we will use an argument to rule out any expected behavior from the dynamics when $\bx_t^C=\bx_t^B$. 
  Indeed, using Lemma~\ref{lem:variationAtZero} from Appendix~\ref{app:integration} shows that 
  \begin{align*}
    \|\bx_t^C-\bx_t^B\|& = \int_0^{t} d\|\bx_s^C-\bx_s^B\| \\
                       &= \lim_{\epsilon \downarrow 0}\int_0^t \indic(\|\bx_s^C-\bx_s^B\| \ge \epsilon) d\|\bx_s^C-\bx_s^B\| \\
                       &\le
                         \lim_{\epsilon\downarrow 0}\int_0^t \eta \ell \indic(\|\bx_s^C-\bx_s^B\|\ge \epsilon) (\|\bx_s^M-\bx_s^B\|+\|\bx_s^C-\bx_s^B\|)ds \\
    &\le  \int_0^t \eta \ell  (\|\bx_s^M-\bx_s^B\|+\|\bx_s^C-\bx_s^B\|)ds.
  \end{align*}

  Thus Gronwall's inequality implies that
\begin{equation*}
\|\bx_t^C-\bx_t^B\| \le \eta \ell \int_0^t e^{\eta \ell (t-s)} \|\bx_{s}^M-\bx_{s}^B\| ds 
\end{equation*}

Taking expectations and using Lemma~\ref{lem:dudley} gives the desired bound:
\begin{align*}
  \bbE[\|\bx_t^C-\bx_t^B\|] &\le \eta \ell \int_0^t e^{\eta \ell (t-s)}
                              \left( c_{\ref{aveTanakaLin}} \eta s^{1/2} + c_{\ref{aveTanakaRoot}} \eta^{1/2} s^{1/4} +c_{\ref{aveTanakaTQ}}\eta s^{3/4}
                              \right)ds \\
                            &\le \eta \ell\left( c_{\ref{aveTanakaLin}} \eta t^{1/2} + c_{\ref{aveTanakaRoot}} \eta^{1/2} t^{1/4}
+c_{\ref{aveTanakaTQ}}\eta t^{3/4}
                              \right) \int_0^te^{\eta \ell (t-s)}ds \\
  &= \left( c_{\ref{aveTanakaLin}} \eta t^{1/2} + c_{\ref{aveTanakaRoot}} \eta^{1/2} t^{1/4}+  c_{\ref{aveTanakaTQ}}\eta t^{3/4}\right) (e^{\eta \ell t}-1).
\end{align*}
\hfill$\blacksquare$

\section{Bounding the Constants}
\label{sec:constants}

In this expression, we bound the size of the constants based on the problem data. First we get simplified bounds on the constants from Proposition~\ref{prop:W1contract}. Then we use this result prove Proposition~\ref{prop:constants}, which bounds the overall constants for the algorithm. 

\begin{lemma}
  \label{lem:contractionConstants}
  If $D^2\ell\beta < 8$ and $a = \frac{4}{D^2\beta}$, then
  $$a \ge \frac{\ell}{2} \quad \textrm{and} \quad c_{\ref{W1mult}} \le
  \frac{2e}{\left(1-\frac{D^2\ell\beta}{8}\right)^2}.
  $$

  Otherwise, if $a = \frac{D^2\ell^2\beta}{16}\left(1-\tanh^2\left(
      \frac{D^2\ell\beta}{8}\right)\right)$, then
  \begin{equation*}
    a \ge \frac{D^2\ell^2 \beta}{16}\exp\left(
      -\frac{D^2\ell\beta}{4}
    \right) \quad \textrm{and} \quad
    c_{\ref{W1mult}} \le \frac{4}{D^2\ell\beta} \exp\left(\frac{D^2\ell\beta}{2} \right).
  \end{equation*}
\end{lemma}

\begin{proof}
  For consider the case that $D^2\ell\beta < 8$ and $a =
  \frac{4}{D^2\beta}$. The lower bound on $a$ is derived by combining
  the inequality $D^2\ell \beta <8$ with the expression for $a$. To
  derive the upper bound on $c_{\ref{W1mult}}$, first note that 
  \begin{equation*}
    D\omega_N \xi = \frac{D^2 \ell\beta}{8} <1
  \end{equation*}
  and so the numerator is bounded by $e^{D\omega_N \xi} \le e$.

  Now we compute a lower bound on the denominator. By the choice of
  $a$, we have that $D\omega_N=1$ and $\xi =
  \frac{D^2\ell\beta}{8}<1$.
  We use the fact that $\sin\left(\sqrt{1-\xi^2}\right)<
    \sqrt{1-\xi^2}$ to give
  \begin{equation*}
    \cos\left(\sqrt{1-\xi^2}\right)-\frac{\xi}{\sqrt{1-\xi^2}}
    \sin\left(
      \sqrt{1-\xi^2}
      \right) \ge \cos\left(\sqrt{1-\xi^2}\right)-\xi
    \end{equation*}
    Then we use the elementary bound
    \begin{equation*}
      \cos(\theta)=\cos(0)-\int_0^\theta \sin(t)dt \ge 1-\int_0^\theta
      t dt = 1-\frac{\theta^2}{2}
    \end{equation*}
    to give
    \begin{equation*}
      \cos\left(\sqrt{1-\xi^2}\right)-\frac{\xi}{\sqrt{1-\xi^2}}
    \sin\left(
      \sqrt{1-\xi^2}
    \right) \ge 1 - \frac{1}{2}\left(
      1-\xi^2
    \right)-\xi = \frac{1}{2}(1-\xi)^2.
  \end{equation*}
  Combining the bounds for the numerator and the denominator, along
  with the fact that $\xi=\frac{D^2\ell\beta}{8}$ gives the desired
  bound on $c_{\ref{W1mult}}$.

  Now consider the case that $a =
  \frac{D^2\ell^2\beta}{16}\left(1-\tanh^2\left(
      \frac{D^2\ell\beta}{8}\right)\right).$ For simpler notation, set
  $x=\frac{D^2\ell\beta}{8}$ so that $a =
  \frac{\ell}{2}x\left(1-\tanh^2(x)\right).$ Then the exponential
  decay factor can be bounded using the fact that:
  \begin{equation*}
    x\left(1-\tanh^2(x)\right) = \frac{4x}{\left(e^x+e^{-x}\right)^2}
    \ge xe^{-2x}.
  \end{equation*}

  Now we bound $c_{\ref{W1mult}}$. Using the definition of $x$, the numerator is given by $e^x$.

  Now we derive a lower bound on
  the denominator of $c_{\ref{W1mult}}$. Let $y = D\omega_N\sqrt{\xi^2-1}$. Plugging in the expressions for $\xi$,
  $a$, and $x$ shows that
  \begin{equation*}
    y = D\omega_N\sqrt{\xi^2-1}=x\tanh(x) \quad \textrm{and}\quad
    \frac{\sqrt{\xi^2-1}}{\xi}=\tanh(x).
  \end{equation*}
  Then the denominator of $c_{\ref{W1mult}}$ can be expressed as
  \begin{equation}
    \label{eq:hyperbolicDenominator}
    \cosh(y)-\frac{\sinh(y)}{\tanh(x)}=\frac{\cosh(y)}{\tanh(x)}\left(
      \tanh(x)-\tanh(y)
    \right) \ge \frac{\tanh(x)-\tanh(y)}{\tanh(x)},
  \end{equation}
  where the inequality follows from the fact that $\cosh(y)\ge 1$.

  Using the fact that $\frac{d}{dx}\tanh(x)=1-\tanh^2(x)$ and that
  $\tanh$ is monotonically increasing gives the bound:
  \begin{align}
    \nonumber
    \tanh(x)-\tanh(y)&=\int_y^x\left(
      1-\tanh^2(z)
                       \right)dz \\
    \nonumber
                     &\ge (x-y) (1-\tanh^2(x))\\
    \nonumber
                     &=x(1-\tanh(x))^2(1+\tanh(x)) \\
    \label{eq:tanhGrowth}
    &=\frac{8xe^{-2x}}{\left(e^{x}+e^{-x}\right)^3}.
    \end{align}
    The second-to-last line uses the fact that $y=x\tanh(x)$, while
    the last line uses that
    $\tanh(x)=\frac{e^{x}-e^{-x}}{e^x+e^{-x}}$. 

    Combining \eqref{eq:hyperbolicDenominator} and
    \eqref{eq:tanhGrowth} shows that
    \begin{equation*}
      \cosh(y)-\frac{\sinh(y)}{\tanh(x)}\ge
      \frac{8x
        e^{-2x}}{\left(e^x+e^{-x}\right)^2\left(e^x-e^{-x}\right)}\ge 2xe^{-5x}.
    \end{equation*}
    Combining the numerator and denominator bounds shows that
    \begin{equation*}
      c_{\ref{W1mult}} \le \frac{e^{4x}}{2x}.
    \end{equation*}
    Plugging in the expression for $x$ gives the result. 
  \end{proof}

\paragraph*{Proof of Proposition~\ref{prop:constants}}
Collecting the constant definitions from the proofs above gives the following relations, in addition to the definitions of $c_{\ref{W1mult}}$ and $a$:
\begin{align*}
  c_{\ref{globalContract}} &= c_{\ref{W1mult}}D \\
c_{\ref{globalConst}}&=c_{\ref{AtoC}}+c_{\ref{CtoM}} \\
  c_{\ref{AtoC}}&=2^{1/4} \left( c_{\ref{tanakaRt}} + c_{\ref{tanakaConst}} \right) e^\ell \left(1+\frac{c_{\ref{W1mult}}}{1-e^{-a/2}}\right)  \\
  c_{\ref{CtoM}} &=  (c_{\ref{aveTanakaLin}} + c_{\ref{aveTanakaRoot}} + c_{\ref{aveTanakaTQ}}) e^{\ell} \left(1+\frac{c_{\ref{W1mult}}}{1-e^{-a/2}}\right) \\
      c_{\ref{tanakaRt}}  &=  \sqrt{2 \left( \frac{u+\ell D}{2r} +\frac{n \sigma}{\sqrt{2}r} +\frac{2n\sqrt{2}}{r\sqrt{\beta}}\right)
                          \left(
        \frac{n}{\beta} + Du + 2Dn\sigma
        \right)}
                       \\
    c_{\ref{tanakaConst}} &=
                            \sqrt{2 \left(Du + 2n\sigma + \frac{n}{\beta} \right)} + D \sqrt{\frac{u+\ell D}{2r} +\frac{n \sigma}{\sqrt{2}r} +\frac{2n\sqrt{2}}{r\sqrt{\beta} } } \\
      c_{\ref{aveTanakaLin}} &=
                             2\sigma \sqrt{n} \\
    c_{\ref{aveTanakaRoot}} &=
\sqrt{\frac{64 n \sigma D\sqrt{2\pi} }{r} } \\
    c_{\ref{aveTanakaTQ}} &=
\sqrt{\frac{128 n \sigma\sqrt{2\pi}}{r} \left(
        \frac{n}{\beta} + Du + 2Dn\sigma
    \right)} 
\end{align*}

Since neither $c_{\ref{W1mult}}$ nor $a$ depend on the state dimension, $n$, we can see that the constants grow linearly with $n$. 

In the case of $D^2\ell\beta <8$ and $a=\frac{4}{D^2\beta}$, Lemma~\ref{lem:contractionConstants} implies that $(1-e^{-a/2})^{-1}\le (1-e^{-\ell/4})^{-1}$. So, the only way for the constants to become large as $\beta$ varies is for $\beta^{-1}$ or $\left(1-\frac{D^2\ell \beta}{8}\right)^{-1}$ to approach $\infty$. In particular the terms that can go to $\infty$ are $\beta^{-1/4}$ and $\left(1-\frac{D^2\ell \beta}{8}\right)^{-2}$ in this case.

Now consider the case that
  \begin{align*}
    a&=\frac{D^2\ell^2\beta}{16}\left(1-\tanh^2\left(
       \frac{D^2\ell\beta}{8}\right)\right) \\
    c_{\ref{W1mult}}&= \frac{e^{D\omega_N \xi}}{\cosh(D\omega_N\sqrt{\xi^2-1})-\frac{\xi}{\sqrt{\xi^2-1}}\sinh(D\omega_N\sqrt{1-\xi^2})}.
  \end{align*}

  The general lower bound on $a$ is taken directly from Lemma~\ref{lem:contractionConstants}. 

  The main
 term that remains to be bounded is $\frac{1}{1-e^{-a/2}}$. To perform this bound, we first note that for all $y> 0$,
  \begin{equation}
  \label{eq:invBound}
  \frac{1}{1-e^{-y}}\le \max\left\{\frac{2}{y},\frac{1}{1-e^{-1}} \right\}.
  \end{equation}
  Indeed, the left side is monotonically decreasing, and so for all $y\ge 1$, the bound
$$
\frac{1}{1-e^{-y}}\le \frac{1}{1-e^{-1}} 
$$
holds.

Now, we use the elementary bound that for all $y\ge 0$, $e^{-y}\le 1-y+\frac{1}{2}y^2$. This inequality appears in \cite{lattimore2019bandit} without proof, but can be proved by showing that $e^y\left(1-y+\frac{1}{2}y^2\right)$ is monotonically increasing. In particular, when $0<y\le 1$, we have that 
\begin{equation*}
  \frac{1}{1-e^{-y}}\le \frac{1}{y\left(1-\frac{1}{2}y\right)} \le \frac{2}{y}.
\end{equation*}
Combining the bounds for $y\ge 1$ an $0<y\le 1$ gives (\ref{eq:invBound}). 

So, now combining (\ref{eq:invBound}) with the results of Lemma~\ref{lem:contractionConstants} shows that
\begin{equation*}
  \frac{c_{\ref{W1mult}}}{1-e^{-a/2}} \le  \frac{4}{D^2\ell\beta} \exp\left(\frac{D^2\ell\beta}{2} \right)
  \max\left\{
  \frac{32}{D^2\ell^2\beta}\exp\left(
    \frac{D^2\ell \beta}{4}
  \right),\frac{1}{1-e^{-1}}
  \right\}
\end{equation*}

Then combining this above bound with the various expressions for the constants shows that there is a polynomial $p$ such $c_{\ref{globalContract}}$ and $c_{\ref{globalConst}}$ can be bounded by
$$
c_i\le 
p(\beta^{-1/4})
\exp\left(\frac{3 D^2\ell \beta}{4}\right)
$$
\hfill$\blacksquare$

\section{Near-Optimality of Gibbs Distributions}
\label{sec:nearOpt}

In this appendix, we prove Proposition~\ref{prop:suboptimality}, 
which states that the algorithm can produce near-optimal samples, provided that $\beta$ is sufficiently large. This proposition depends on an elementary  result on the properties of Gibbs distributions constrained to $\cK$, shown next. 

\begin{lemma}
  \label{lem:KLfromUniform}
  For any function $g:\cK\to \bbR$, let  $\pi_g$ be the probability
  measure defined by $\pi_g(A) =\frac{\int_A e^{-g(x)}dx}{\int_{\cK}
    e^{-g(y)}dy}$. In particular, $\pi_0$ corresponds to the uniform
  measure. If $g$ is $\ell$-lipschitz, then the KL divergence of
  $\pi_g$ from the uniform measure is bounded by:
  \begin{multline}
    \nonumber
    0 \le \KL(\pi_g,\pi_0)\le
    \\
    \min_{x\in\cK} g(x) -
    \bbE_{\pi_g}[g(\bx)] +
    n\log\left(
      \max\left\{\frac{2}{r},\frac{(r+\sqrt{r^2+D^2})\ell}{r\log 2}\right\}
    \right) + 
    \log(2D^n).
  \end{multline}
\end{lemma}
\begin{proof}
  The lower-bound on the KL divergence is standard
  \cite{cover2012elements,gray2011entropy}. Now we prove the upper bound. 
  
  Say $x^\star$ minimizes $g(x)$ over $\cK$. A minimizer exists
  because $g$ is Lipschitz and $\cK$ is compact. Multiplying the
  numerator and denominator of the definition of $\pi_g$ by
  $e^{g(x^\star)}$ gives
  $$\pi_g(A) =\frac{\int_A e^{g(x^\star)-g(x)}dx}{\int_{\cK}
    e^{g(x^\star)-g(y)}dy}.$$

  Note that $\pi_0(dx) = \frac{dx}{\vol(\cK)}$. Thus, the definition
  of KL divergence gives:
  \begin{equation*}
    \KL(\pi_g,\pi_0) = \bbE_{\pi_g}[g(x^\star)-g(\bx)] +
    \log\left(\vol(\cK)\right)-
    \log\left(\int_{\cK} e^{g(x^\star)-g(x)}dx \right)
  \end{equation*}
  Note that $\cK$ is contained in a ball of radius $D$, so that
  $\vol(\cK)\le D^n \frac{\pi^{\frac{n}{2}}}{\Gamma\left(\frac{n}{2}+1 \right)}$.
  
  So, the desired upper bound is obtained by providing a lower bound
  on $\int_{\cK}e^{g(x^\star)-g(x)}dx$.
  Note that
  \begin{equation*}
    0 \ge g(x^\star) - g(x) \ge -\ell\|x-x^\star\|
  \end{equation*}
  Also, note that $e^{-\ell \|x-x^\star\|}\ge \frac{1}{2}$ if and only
  if $\|x-x^\star\|\le \frac{\log 2}{\ell}$.

  Set $\epsilon = \frac{\log 2}{\ell}$ and let
  $\cB_{x^\star}(\epsilon)$ be the ball of radius $\epsilon$ centered
  at $x^\star$. Then for any $\cS\subset \cK\cap
  \cB_{x^\star}(\epsilon)$ we have that
  \begin{equation*}
    \int_{\cK} e^{g(x^\star)-g(x)}dx \ge \frac{1}{2}\vol\left(
      \cK\cap
  \cB_{x^\star}(\epsilon)
    \right) \ge \frac{1}{2}\vol(\cS)
  \end{equation*}
  We will show that $\cK\cap \cB_{x^\star}(\epsilon)$ always contains
  a ball of radius
  $\min\{\frac{r}{2},\frac{r\epsilon}{r+\sqrt{r^2+D^2}}\}$. The lemma
  then follows by using the fact that a ball of radius $\rho$ has
  volume given by
  $\frac{\pi^{\frac{n}{2}}}{\Gamma\left(\frac{n}{2}+1\right)}\rho^n$. Note
   that the constant factors of
  $\frac{\pi^{\frac{n}{2}}}{\Gamma\left(\frac{n}{2}+1\right)}$ cancel
  in the bound.

  To find the desired ball, we consider three cases: 1) $0\notin
  \cB_{x^\star}(\epsilon)$, 2) $0\in \cB_{x^\star}(\epsilon)$ and
  $\epsilon \le r$, and 3) $0\in \cB_{x^\star}(\epsilon)$ and
  $\epsilon > r$. 

  When $0\notin \cB_{x^\star}(\epsilon)$, we construct the desired
  ball from the geometry of Fig.~\ref{fig:cone}. Without loss of
  generality, we can assume that $x^\star = -\|x^\star\| e_1$, where
  $e_1$ is the first standard unit vector. Also, since $0\notin
  \cB_{x^\star}(\epsilon)$, we must have that
  $\|x^\star\|>0$. Consider the convex set defined by:
  \begin{subequations}
    \label{eq:coneSubset}
    \begin{gather}
      \|x-x^\star\|\le \epsilon \\
    -\|x^\star\|\le
    x_1\le  0 \\
    \label{eq:cone}
 \sqrt{\sum_{i=2}^n x_i^2}\le r + \frac{r}{\|x^\star\|}x_1
\end{gather}
\end{subequations}
The set defined by \eqref{eq:coneSubset} is a subset of
$\cK\cap\cB_{x^\star}(\epsilon)$. The angle between the $x^\star$ and
the conic constraint boundary, from \eqref{eq:cone}, is given by $\theta =
\tan^{-1}\frac{r}{\|x^\star\|}$. For any $d>0$, the largest ball
centered at
$(-\|x^\star\|+d)e_1$ which fits into the conic set from
\eqref{eq:cone} has radius $d\sin\theta$. The largest such ball that
is also contained in $\cB_{x^\star}(\epsilon)$ is found by setting
$d+d\sin\theta=\epsilon$. Plugging the definitions of $d$ and $\theta$
shows that the corresponding ball has radius $\rho$, which satifies
\begin{equation*}
  \frac{r\epsilon}{r+\sqrt{r^2+D^2}}\le \rho =
  \frac{r\epsilon}{r+\sqrt{r+\|x^\star\|^2}} < \frac{\epsilon}{2}.
\end{equation*}

  \begin{figure}
    \centering
    \begin{tikzpicture}
      \def\r{2}
      \def\D{3}
      \def\eps{2}
      \draw[] (0,0) circle (\eps);
      \draw
      (\D,0) node[label={[xshift=-.2cm, yshift=-.1cm]0}] {} coordinate (origin) --
      (0,0) node[left] {$x^\star$} coordinate (x) -- (\D,\r)
      coordinate (top)
      ;
      \draw
      (0,0) -- (\D,-\r) -- (\D,\r);

      \draw[|-|] (\D+.2,0) -- node[right] {$r$} (\D+.2,\r);


      \draw[|-|] (0,\r+.2)-- node[above] {$\epsilon$} (\eps,\r+.2);
      
      \draw[dotted] ($({\eps*sqrt(\r*\r +\D*\D)/(\r+sqrt(\r*\r
        +\D*\D))},0)$) circle ({\eps * \r/(\r + sqrt(\r*\r
        +\D*\D))});

      \draw[fill] (0,0) circle (.05);
      \draw[fill] ($({\eps*sqrt(\r*\r +\D*\D)/(\r+sqrt(\r*\r
        +\D*\D))},0)$) circle (.05); 
    \end{tikzpicture}
    \caption{\label{fig:cone} {\bf Case of $0\notin
        \cB_{x^\star}(\epsilon)$.} Using elementary trigonometry the
      largest inscribed circle can be calculated.}
  \end{figure}
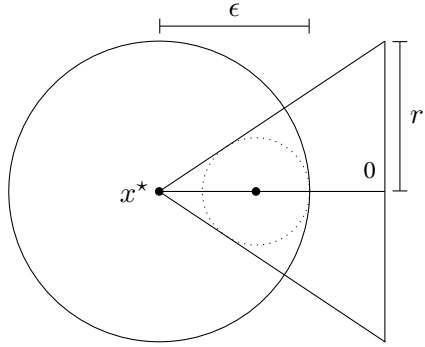

  Now consider the case that $0\in\cB_{x^\star}(\epsilon)$ and
  $\epsilon \le r$. It follows that $x^\star \in \cB_0(r)$. Then
  applications of the triangle inequality show that
  $\cB_{x^\star/2}(\epsilon/2) \subset \cB_0(r)\cap
  \cB_{x^\star}(\epsilon)\subset \cK\cap
  \cB_{x^\star}(\epsilon)$.
  Thus, a
  ball of radius $\epsilon/2> \frac{r\epsilon}{r+\sqrt{r+D^2}}$ has
  been constructed in  $\cK\cap \cB_{x^\star}(\epsilon)$.

  Finally, consider the case that $0\in\cB_{x^\star}(\epsilon)$ and
  $\epsilon > r$. If $\|x^\star\| \ge r/2$, then $\cB_{\frac{r x^\star}{2\|x^\star\|}}(r/2)\subset\cB_0(r)\cap
  \cB_{x^\star}(\epsilon)\subset \cK\cap
  \cB_{x^\star}(\epsilon)$. Otherwise, if $\|x^\star \|< r/2$, then
  $\cB_{0}(r/2) \subset\cB_0(r)\cap
  \cB_{x^\star}(\epsilon)\subset \cK\cap
  \cB_{x^\star}(\epsilon)$. In either case, a
  ball of radius $r/2$ has
  been constructed in  $\cK\cap \cB_{x^\star}(\epsilon)$.
\end{proof}

\paragraph*{Proof of Proposition~\ref{prop:suboptimality}}

Recall that $\bar f$ is $u$-Lipschitz, so that $\beta \bar f$ is $\beta u$-Lipschitz.
Assume that $\tilde\bx$ is drawn according to $\pi_{\beta \bar f}$. So, applying Lemma~\ref{lem:KLfromUniform} to $\beta\bar f$ and dividing by $\beta$ implies that
\begin{equation}
  \label{eq:nearOptGibbs}
\bbE[\bar f(\tilde \bx)]\le \min_{x\in\cK} \bar f(x) + 
    \frac{n}{\beta}\log\left(2D
      \max\left\{\frac{2}{r},\frac{(r+\sqrt{r^2+D^2})u\beta}{r\log 2}\right\}
    \right).
\end{equation}

 Let $\bx_k$ be the $k$-th iterate of the algorithm.
 \begin{align*}
\MoveEqLeft
   \bbE[f(\bx_k)] \\
   &\overset{\textrm{Kantorovich Duality}}{\le}  \bbE_{\pi_{\beta \bar f}}[\bar f(\bx)] + uW_1(\cL(\bx_k),\pi_{\beta \bar f}) \\
  &\overset{\eqref{eq:nearOptGibbs}}{\le} \min_{x\in\cK} \bar f(x) + 
uW_1(\cL(\bx_k),\pi_{\beta \bar f})+\frac{n\log(c_{\ref{subOpt}}\max\{1,\beta\})}{\beta },
 \end{align*}
 where $c_{\ref{subOpt}} = 2D\max\left\{\frac{2}{r},\frac{(r+\sqrt{r^2+D^2})u}{r\log 2}\right\}$.

 Now we will show how to tune the parameters to achieve an average suboptimality of $\epsilon$.

 First, we choose $\beta$ so that
 $
\frac{n\log(c_{\ref{subOpt}}\max\{1,\beta\})}{\beta } \le \frac{\epsilon}{2}.
$
Without loss of generality, assume that $\beta \ge 1$. Set $x=\log(c_{\ref{subOpt}} \beta)$, so that $\beta = c_{\ref{subOpt}}^{-1}e^x$ and the required bound becomes:
$$
xe^{-x}\le \frac{c_{\ref{subOpt}}\epsilon}{2n}
$$

Fix any $\lambda\in(0,1)$. Then the maximum value of $xe^{-(1-\lambda)x}$ occurs at $x=(1-\lambda)^{-1}$, so that
for all $x\in\bbR$:
\begin{equation}
  \label{eq:removeLog}
xe^{-x}\le \frac{1}{(1-\lambda) e}e^{-\lambda x}.
\end{equation}
So, it suffices to set $e^{-\lambda x} \le \frac{c_{\ref{subOpt}}\epsilon (1-\lambda)e}{2n}$. Plugging in the definition of $x$ and re-arranging shows that a sufficient condition for  $
\frac{n\log(c_{\ref{subOpt}}\beta)}{\beta } \le \frac{\epsilon}{2}
$ is given by:
\begin{equation}
\label{eq:betaRequirement}
\beta \ge c_{\ref{subOpt}}^{-1}\left(
  \frac{2n}{c_{\ref{subOpt}}(1-\lambda) \epsilon e}
  \right)^{1/\lambda}.
\end{equation}

  Now, for fixed $\beta\ge 1$, the bounds from Theorem~\ref{thm:nonconvexLangevin} and Proposition~\ref{prop:constants} to give that
  \begin{align*}
    W_1(\cL(\bx_k),\pi_{\beta \bar f})
    &
      \left(c_{\ref{globalContract}}+\frac{c_{\ref{globalConst}}}{(4a)^{1/4}}\right) T^{-1/4}(\log T)^{1/2} \\
    &\le p(1) e^{\frac{3 D^2\ell \beta}{4}}
      \left(1+\frac{e^{\frac{D^2\ell\beta}{16}}}{4^{1/4}c_{\ref{smallA}}}\right) T^{-1/4}(\log T)^{1/2} \\
    &\le p(1) \left(1+\frac{1}{4^{1/4}c_{\ref{smallA}}}\right) \exp\left(\frac{13 D^2\ell \beta}{16} \right) T^{-1/4}(\log T)^{1/2} 
  \end{align*}

  Similar to the derivation of (\ref{eq:removeLog}) we have for all $\delta \in (0,1/2)$ and all $T>0$:
  \begin{align*}
    T^{-1/4}(\log T)^{1/2} \le \sqrt{\frac{T^{-\frac{1}{2}+\delta}}{e\delta}}
  \end{align*}

  Thus, to have $uW_1(\cL(\bx_k),\pi_{\beta \bar f})\le \frac{\epsilon}{2}$, it suffices to have 
  $$
  T^{-\frac{1}{2}+\delta} \le e\delta \left(\frac{\epsilon}{2}\right)^2 \left(p(1) \left(1+\frac{1}{4^{1/4}c_{\ref{smallA}}}\right) \exp\left(\frac{13 D^2\ell \beta}{16} \right) \right)^{-2} =:\hat\epsilon,
  $$
  which occurs whenever
  $$
  T\ge \frac{1}{\hat\epsilon^{\frac{2}{1-2\delta}}} 
  $$
\const{timeBound}
In particular, there is a constant, $c_{\ref{timeBound}}$, independent of $\eta$, $\beta$, $\epsilon$, $\lambda$, and $\delta$, such that the bound above holds whenever
\begin{equation}
  \label{eq:TRequirement}
T\ge \frac{c_{\ref{timeBound}}^{\frac{2}{1-2\delta}}}{\epsilon^{\frac{4}{1-2\delta}}} \exp\left(\frac{13 D^2\ell\beta}{4(1-2\delta)}\right).
\end{equation}
The result now follows by combining (\ref{eq:betaRequirement}) and (\ref{eq:TRequirement}), noting that $\frac{4}{1-2\delta}$ can take any value $\rho > 4$ and $1/\lambda$ can take any value $\zeta  >1$. 
 \hfill$\blacksquare$

\section{The Skorokhod Problem}
\label{appsec:skorohod}

This appendix describes basic results on the Skorokhod problem, which is used to construct solutions to reflected SDEs. First we describe some existing theory. Then we present limiting argument that is used to translate results on compact convex sets with smooth boundaries general compact convex sets. 

\subsection{Background}
A classical construction for constraining stochastic processes to remain
in a set is based on the Skorokhod problem, which we describe
below. This will be useful, in particular, for analyzing projected
gradient algorithms in continuous time. 

Let $\cK$ be a convex subset of $\bbR^n$ with non-empty interior.  
Let $w:[0,\infty)\to \bbR^n$
be a piecewise-continuous function with $w_0\in \cK$. For each $x\in\bbR^n$, let
$N_{\cK}(x)$ be the normal cone at $x$. 
Then the functions $x_t$ and
$\phi_t$ solve the \emph{Skorokhod problem} for $w_t$ if the
following conditions hold:
\begin{itemize}
\item $x_t = w_t+\phi_t \in \cK$ for all $t\in [0,T)$
\item The function $\phi$ has the form $\phi(t) = -\int_0^t v_s
  d\mu(s)$, where $\|v_s\|\in \{0,1\}$ and $v_s\in N_{\cK}(x_s)$ for
  all $s\in [0,T)$, while the measure, $\mu$, satisfies
  $\mu([0,T))<\infty$ for any $T>0$. 
\end{itemize}

For each $w$, the corresponding functions $x_t$ and $\phi_t$ exist and
are unique \cite{tanaka1979stochastic}. 
Note that if $x_t \in
\interior(\cK)$, then $N_{\cK}(x_t)=\{0\}$, and so $v_t = 0$. Thus,
without loss of generality, we can assume that $\mu$ is supported
entirely on the times in which $x_t \in\partial \cK$. 
In many cases, we are primarily interested in $x_t$ and so we will
often refer to $x_t$ as the solution of the Skorokhod problem
corresponding to $w_t$. By existence and
uniqueness, we can view the Skorokhod problem solution as a mapping:
$x=\cS(w)$.


The connection between Skorokhod problems and projection algorithms
becomes more concrete when $w_t$ is piecewise constant. Specifically,
assume that $0= t_0 < t_1 < \cdots < t_{M-1} \le T$ are the jump
points of $w_t$, and let $S_k = [t_k,t_{k+1})$ for $k < M-1$ and
$S_{M-1}=[t_{M-1},T]$. Then $w_t$ can be represented as
\begin{equation*}
  w_t  =\sum_{k=0}^{M-1}w_{t_k}\indic_{S_k}(t).
\end{equation*}

Then the solution of the Skorokhod problem has the form
\begin{equation*}
  x_t = \sum_{k=0}^{M-1} x_{t_k} \indic_{S_k}(t), \quad \phi_t =
  -\int_0^t \sum_{k=0}^{M-1} v_{t_k} d_{k+1}\delta(s-t_k)ds,
\end{equation*}
where $x_0=w_0$, $v_0 = 0$, and
\begin{align*}
  x_{t_{k+1}} &= \Pi_{\cK}(x_{t_{k}}+w_{t_{k+1}}-w_{t_k}) \\
  d_{k+1} &= \|(x_{t_k}+w_{t_{k+1}}-w_{t_k})-x_{t_{k+1}}\| \\
  v_{t_{k+1}} &= \begin{cases}
    0 & x_{t_k}+w_{t_{k+1}}-w_{t_k} \in \cK \\
    \frac{(x_{t_k}+w_{t_{k+1}}-w_{t_k})-x_{t_{k+1}}}{d_{k+1}} & x_{t_k}+w_{t_{k+1}}-w_{t_k} \notin \cK.
    \end{cases}
\end{align*}

In \cite{tanaka1979stochastic}, a construction for the Skorokhod solution for a continuous trajectory, $w$, proceeds as follows. The continuous trajectory is approximated by piecewise constant trajectories of the form $w_{\floor*{ti}/i}$ for positive integers $i$. Then the Skorokhod problems are solved for these discretized trajectories and shown to converge to a unique solution for the original Skorokhod problem for $w$.

The existence of a solution to the Skorokhod problem for arbitrary continuous trajectories can be used to construct unique solutions to reflected stochastic differential equations. 
 In particular, the integrated form of a reflected SDE can be expressed as:
 \begin{equation}
\label{eq:SDE}
\bx_t = \bx_0 + \int_0^t f(s,\bx_s)ds + \int_0^t \sigma(s,\bx_s)d\bw_s -\int_0^t\bv_s d\bmu(s),
\end{equation}
where $\bx_0\in\cK$, $-\int_0^t\bv_s d\bmu(s)$ is the reflection process that ensures that $\bx(t)\in\cK$ for all $t\ge 0$.  Note that $\bx$ is the Skorokhod solution to the process:
$$
\by_t=
\bx_0 + \int_0^t f(\bx_s)ds + \int_0^t \sigma(\bx_s)d\bw_s.
$$

A construction for $\bx_t$ based on Picard iteration was given in \cite{tanaka1979stochastic}. The paper \cite{slominski2001euler}, examines the  Euler scheme defined by: $\bar \bx^m_0=\bx_0$ and for integers $k\ge 0$:
\begin{equation}
\label{eq:skorokhodEuler}
\bar \bx_{(k+1)/m}^m = \Pi_{\cK}\left(\bar \bx_{k/m}^m+\frac{1}{m}f(\bar\bx_{k/m}^m) +
\sigma(\bar \bx_{k/m}^m)(\bw_{(k+1)/m}-\bw_{k/m})\right).
\end{equation}
Then for $t\in[k/m,(k+1)/m)$, we set $\bar\bx_{t}^m =\bar\bx_{k/m}^m$. Corollary 3.3 of \cite{slominski2001euler} shows that $\bar\bx^m$ converges uniformly to $\bx$ on compact subsets of $[0,\infty)$.

\subsection{Approximating the Domain}
In Appendix~\ref{sec:invariance} we show that the distribution $\pi_{\beta \bar f}$ from (\ref{eq:gibbs}) is invariant for the process $\bx^M$. However, many of the arguments are easier when $\cK$ has a smooth boundary. To handle the general case, we examine Skorokhod solutions on smooth approximations of $\cK$ and then use a limiting argument. Here we build the approximation results needed for this argument. The basic idea for this approximation is discussed in Section 2 of \cite{lions1984stochastic}, but not proved explicitly. 

Let $\cK_1\subset \cK_2 \subset \cdots \subset \cK$ be an increasing family of  convex compact sets such that if $S$ is any compact subset of $\mathrm{int}(\cK)$, then $S\subset\cK_i$ for all sufficiently large $i$. 

The approximation results are proved using the following fact about projections. 

\begin{lemma}
  \label{lem:projConverge}
  The functions $\Pi_{\cK_i}$ converge uniformly to $\Pi_{\cK}$ on compact subsets of $\bbR^n$. 
\end{lemma}
\begin{proof}
  We will show that for any $\epsilon>0$, there is an $i$ such that $\|\Pi_{\cK_i}(x)-\Pi_{\cK}(x)\|\le \epsilon$ for all $x\in \bbR^n$ with $\|x\|\le R$. By the monotonicity of $\cK_i$, the result then holds for all $j\ge i$.  Note that if $x\in \cK_i$, then $\Pi_{\cK_i}(x)=\Pi_{\cK}(x)=x$, so we only need to analyze the case that $x\notin\cK_i$.
  
  Let $\delta >0$ and let $S$ be a compact subset of $\mathrm{int}(\cK)$ such that $\dist(x,S)\le \delta $ for all $x\in\cK$. Here $\dist(x,S)$ is the distance function. Such  an $S$ can be chosen as $S=\lambda \cK$ for $\lambda\in (0,1)$ sufficiently close to $1$. Take $\cK_i$ such that $S\subset \cK_i$. This implies, in particular that for any $x\in \cK$, that $\dist(x,\cK_i)\le \delta$.

  Consider any $x\notin \cK_i$. By the distance assumption, there is a point $y\in \cK_i$ such that $\left\|y-\Pi_{\cK}(x)\right\|\le \delta$. The generalized Pythagorean inequality applied to $\|\cdot \|^2$ and $\cK_i$ implies that 
  \begin{equation}
    \label{eq:gpe}
  \|y-x\|^2 \ge \|y-\Pi_{\cK_i}(x)\|^2 + \|x-\Pi_{\cK_i}(x)\|^2.
  \end{equation}
(See \cite{herbster2001tracking} for more on the generalized Pythagorean inequality.) 

Note that $\|x-\Pi_{\cK_i}(x)\|\ge \|x-\Pi_{\cK}(x)\|$, since $\Pi_{\cK_i}(x)\in \cK$. Furthermore, the triangle inequality, followed by the distance assumption on $y$ imply that:
$$
\|y-x\|\le \|y-\Pi_{\cK}(x)\| + \|x-\Pi_{\cK}(x)\| \le \delta + \|x-\Pi_{\cK}(x)\|
$$
Plugging the upper and lower bounds into (\ref{eq:gpe}) shows that
\begin{equation*}
  \delta^2 + 2\delta \|x-\Pi_{\cK}(x)\| + \|x-\Pi_{\cK}(x)\|^2 \ge
  \|y-\Pi_{\cK_i}(x)\|^2 + \|x-\Pi_{\cK}(x)\|^2.
\end{equation*}
Using the fact that $0\in \cK$ and $\|x\|\le R$ shows that $\|x-\Pi_{\cK}(x)\|\le R$. So, rearranging the inequality above and plugging in this bound shows that:
$$
\|y-\Pi_{\cK_i}(x)\| \le \sqrt{\delta^2 + 2\delta R}. 
$$
Also, note that by the triangle inequality, the assumption on $y$, and the inequality above:
$$
\|\Pi_{\cK}(x)-\Pi_{\cK_i}(x)\|
\le \|\Pi_{\cK}(x)-y\| + \|y-\Pi_{\cK_i}(x)\| \le \delta+  \sqrt{\delta^2 + 2\delta R}. 
$$
So the result holds by choosing $\delta$ such that $\delta + \sqrt{\delta^2 + 2\delta R}\le \epsilon$.  
\end{proof}

\begin{lemma}
  \label{lem:sdeApprox} Let $\bx$ and $\bx_i$ be solutions of the reflected SDE from (\ref{eq:SDE}) over domains $\cK$ and $\cK_i$ respectively, with $f(x)$ and $\sigma(x)$ both Lipschitz in $x$. For almost all realizations of the Brownian motion, $\bx_i$ converges uniformly to $\bx$ on compact subsets of $[0,\infty)$.
\end{lemma}
\begin{proof}
  In the proof we denote the trajectories like $\bx(t)$ and $\bx_i(t)$ to reduce the complexity of the subscripts and superscripts.
  
  Consider the Euler approximation from (\ref{eq:skorokhodEuler}). For almost all realizations of the Brownian motion, the resulting solution converges uniformly on compacts to $\bx$. Similarly, for each $\cK_i$, the corresponding Euler scheme converges uniformly on compacts to $\bx_i$ for almost all Brownian motion realizations. Now since the intersection of a countable collection of almost sure events is again an almost sure event, we have for almost all Brownian motion realizations, all of the corresponding Euler schemes converge uniformly on compacts.

  Let $\bw$ be a realization for which all of the corresponding Euler schemes converge uniformly on compacts. Fix $T>0$ and let $\bar \bx_i^m$ be the Euler approximations of $\bx_i$. Corollary 3.3 of \cite{slominski2001euler} gives a convergence rate for the Euler scheme which implies that there is a constant $c>0$ such that
  $$
  \forall i, \sup_{t\in[0,T]} \|\bar\bx_i^m(t)-\bx_i(t)\|\le cm^{-1/5}
  \quad\textrm{and}\quad
    \sup_{t\in[0,T]} \|\bar\bx^m(t)-\bx(t)\|\le cm^{-1/5}.
    $$

   So fix $\epsilon >0$ and choose $m$  sufficiently large so that 
   $$
  \forall i, \sup_{t\in[0,T]} \|\bar\bx_i^m(t)-\bx_i(t)\|\le \epsilon
  \quad\textrm{and}\quad
    \sup_{t\in[0,T]} \|\bar\bx^m(t)-\bx(t)\|\le \epsilon.
    $$

    Now we will show that $i$ can be chosen so that $\sup_{t\in[0,T]}\|\bar \bx_i^m(t)-\bx^m(t)\|\le \epsilon$. If we can show this, the result will follow by the triangle inequality. 

    Let
 $d = \sup_{s,t\in[0,T]}\|\bw_s-\bw_t\|$. Since $f$ and $\sigma$ are continuous, they are bounded on $\cK$.  It follows that all of the arguments of the projection used in the Euler schemes are bounded in norm by:
  $$
  D + \sup_{x\in\cK} \|f(x)\|+ \sup_{x\in\cK}\|\sigma(x)\|_2 d.
  $$
  Here $\|\cdot \|_2$ is the matrix $2$-norm. 

  Thus, for any fixed $m$, 
  Lemma~\ref{lem:projConverge} implies that all of the projections converge uniformly as $i\to\infty$, which in turn implies that $\bar \bx_i^m$ converges to $\bar \bx^m$ uniformly on $[0,T]$. In particular, we can choose $m$ such that $\sup_{t\in[0,T]} \|\bx(t)-\bar\bx^m(t)\|\le  \epsilon$, and the proof is complete.  
\end{proof}

\section{Invariance of the Gibbs Distribution}
\label{sec:invariance}

Here we prove a basic result that $\pi_{\beta \bar f}$ is invariant for $\bx^M$. Our proof extends the methodology from Lemma 2.1 of~\cite{harrison1987multidimensional}, which examines the case that $\bar f$ is affine and the boundary is smooth. \cite{bubeck2018sampling} gives a brief outline of the analysis when $\bar f$ is convex and $\cK$ is a general compact convex set. The basic idea follows through in the more general case in which $\bar f$ is only assumed to be differentiable.

\begin{lemma}
  \label{lem:invariance}
 The measure $\pi_{\beta \bar f}$ is a stationary distribution for (\ref{eq:averagedLangevin}). 
\end{lemma}

\begin{proof}
  We first assume that $\cK$ has a smooth boundary. Later, we will use a limiting argument to show that the result still holds for general compact convex $\cK$. 
  
  The generator associated with $\bx^M$ on the interior of $\cK$ is given by:
  $$
   Lg(x) = -\eta \nabla \bar f(x)^\top \nabla g(x) + \frac{\eta}{\beta}(\Delta g)(x),
   $$
   where $\Delta$ is the Laplacian operator. (In this proof we will drop the subscript of $x$ from the gradient operators, since $\bz_k$ does not influence $\bx^M$.)

   Define the diffusion operator $P_t$ by:
   $$
   (P_tg)(x) = \bbE[g(\bx_t^M)|\bx_0=x].
   $$
   To show invariance, it suffices to show that 
   for all $g\in L_2(\pi_{\beta \bar f})$ and all $t> 0$ the follow holds:
   \begin{equation}
     \label{eq:stationaryTest}
   \int_{\cK} g(x)d\pi_{\beta \bar f}(x) = \int_{\cK} (P_tg)(x)d\pi_{\beta \bar f}(x)
   \end{equation}

   Now, since the set of differentiable functions is dense in $L_2(\pi_{\beta \bar f})$, we can assume without loss of generality that $g$ is differentiable. In the case that $g$ is differentiable, Theorem 6.31 of \cite{gilbarg1998elliptic} shows that there is a unique twice-differentiable $h$ such that
   \begin{subequations}
     \begin{align}
       \label{eq:resolvedSol}
     (Lh)(x)-\lambda h(x) &= -g(x) \quad \forall x\in \mathrm{int}(\cK) \\
     \label{eq:tangent}
     \nabla h(x)^\top v &= 0 \quad \forall x\in \partial \cK \textrm{ and } \forall v\in N_{\cK}(x).
   \end{align}
   \end{subequations}
   Note that since $\partial\cK$ is smooth, $N_{\cK}(x)$ is a half-line for all $x\in\partial\cK$.

   Let
   $$
   \bq_t = e^{-\lambda t} h(\bx_t^M)+\int_0^t e^{-\lambda s} g(\bx_s^M)ds.
   $$
   Then It\^o's formula combined with (\ref{eq:tangent}) followed by (\ref{eq:resolvedSol}) gives:
   \begin{align*}
     \MoveEqLeft[0]
     d\bq_t \\
     &= e^{-\lambda t}(-\lambda h(\bx_t^M)-\eta \nabla \bar f(\bx_t^M)^\top \nabla h(\bx_t^M) + \frac{\eta}{\beta}\Delta h(\bx_t^M) + g(\bx_t^M))dt + \sqrt{\frac{2\eta}{\beta}} \nabla g(\bx_t^M)^\top d\bw_t \\
     &= \sqrt{\frac{2\eta}{\beta}} \nabla g(\bx_t^M)^\top d\bw_t.
   \end{align*}
   So, in particular, $\bq_t$ is a martingale.

   If $x=\bx_0^M$, then
   \begin{align}
     \nonumber
     h(x)&=\bq_0 \\
     \nonumber
         &=\lim_{t\to\infty} \bbE[\bq_t|\bx_0=x] \\
     \nonumber
         &= \bbE\left[\int_0^{\infty} e^{-\lambda t} g(\bx_t^M)dt\middle| \bx_0^M=x\right] \\
     \label{eq:operatorLaplace}
     &=\int_0^{\infty} e^{-\lambda t} (P_tg)(x) dt.
   \end{align}
   The last equality follows from Fubini's theorem, which is justified by the fact that $g$ is differentiable, and so the integrand is bounded on $\cK$.

   Taking the Laplace transform of both sides of (\ref{eq:stationaryTest}) gives the condition:
  \begin{equation}
    \label{eq:laplaceTest}
     \lambda^{-1} \int_{\cK}g(x) d\pi_{\beta \bar f}(x)=\int_{\cK}\int_0^{\infty}e^{-\lambda t} (P_tg)(x) dt d\pi_{\beta \bar f}(x).
   \end{equation}
   Note that Fubini's theorem was again used to switch the order of integrals on the right. Now, uniqueness of Laplace transforms implies that (\ref{eq:stationaryTest}) holds for all $t>0$ if and only if (\ref{eq:laplaceTest}) holds for all $\lambda >0$.

   Using (\ref{eq:resolvedSol}), the left side of (\ref{eq:laplaceTest}) becomes:
   \begin{equation*}
      \lambda^{-1} \int_{\cK}g(x) d\pi_{\beta \bar f}(x)=
     \lambda^{-1}\int_{\cK}(\lambda h(x)-Lh(x))d\pi_{\beta \bar f}(x)
   \end{equation*}
   
   Using (\ref{eq:operatorLaplace}), the right side of
 (\ref{eq:laplaceTest}) becomes:
\begin{equation*}
  \int_{\cK}\int_0^{\infty}e^{-\lambda t} (P_tg)(x) dt d\pi_{\beta \bar f}(x)
=
    \int_{\cK} h(x) d\pi_{\beta \bar f}(x). 
  \end{equation*}
  Thus, we see that (\ref{eq:laplaceTest}) holds if and only if
  \begin{equation}
    \label{eq:instantaneous}
    \int_{\cK}Lh(x) d\pi_{\beta \bar f}(x)=0.
  \end{equation}

  Using the specific form of $L$ and $\pi_{\beta \bar f}$, we see that (\ref{eq:instantaneous}) holds if and only if
\begin{equation}
  \label{eq:instantaneousExplicit}
  \int_{\cK}\left(
    -\eta \nabla \bar f(x)^\top \nabla h(x)+\frac{\eta}{\beta} \Delta h(x)
  \right)e^{-\beta \bar f(x)} dx =  0
\end{equation}

We will prove (\ref{eq:instantaneousExplicit})
via Stokes theorem, which states that $\int_{\cK}d\omega(x)=\int_{\partial \cK}\omega(x)$ for a differential $n-1$ form. Consider the $(n-1)$-form defined by:
$$
\omega(x) = \sum_{i=1}^n(-1)^{i+1}\left(
  \frac{\partial h(x)}{\partial x_i} e^{-\beta \bar f(x)}
\right)\bigwedge_{j\ne i} dx_j
$$
Here the wedge product follows the standard ordering over the integers.

By construction, we have
$$
d\omega(x)= \left(
    -\eta \nabla \bar f(x)^\top \nabla h(x)+\frac{\eta}{\beta} \Delta h(x)
  \right)e^{-\beta \bar f(x)} dx_1 \wedge \cdots \wedge dx_n.
  $$
  Thus, Stokes theorem implies (\ref{eq:instantaneousExplicit}) if and only if
$$
\int_{\partial \cK}\omega(x)=0
$$

To evaluate this integral, we follow a typical construction from the integration of differential forms from \cite{lee2013introduction}. 
Choose a finite open cover of $\cK$ in the subspace topolgy, $U^1,\ldots,U^m$, with corresponding smooth charts, $\phi^1,\ldots,\phi^m$,  and a corresponding partition of unity $\psi^1,\ldots,\psi^m$. The partition of unity has the property that $\psi^i$ is supported in $U^i$. Without loss of generality, we assume that $\phi^i$ preserve the orientation of $\cK$. Furthermore, since $\partial\cK$ is smooth, the neighborhoods and charts can be chosen such that if $U^i\cap \partial \cK\ne \emptyset$, then $\phi^i$ maps $U^i$ to a half-space, $\bbH$:
\begin{align*}
  \phi^i(U^i)&\subset \{y\in\bbR^n | y_n\ge 0\} \\
  \phi^i(U^i\cap \partial \cK) &\subset \{y\in\bbR^n | y_n=0\}.
\end{align*}

As in \cite{lee2013introduction}, the desired integral can be evaluated as
\begin{align*}
  \int_{\partial\cK} \omega(x)&= \sum_{i=1}^M\int_{\partial\cK} \psi^i(x)\omega(x) \\
  &=\sum_{i=1}^M \int_{\partial \bbH^n} (((\phi^i)^{-1})^\star(\psi^i\omega))(y),
\end{align*}
where $((\phi^i)^{-1})^\star$ denotes the pullback operation. 

To evaluate $\int_{\cK} \omega(x)$, it suffices to evaluate this integral over elements of the cover that intersect $\partial\cK$. We will show that each of these elements integrates to zero. To this end,  
let $U$ be a set in the cover with $U\cap \partial \cK$ with associated chart $\phi$ and partition of unity element $\psi$. For compact notation, set
$$
\alpha_i(x) = \psi(x) \frac{\partial h(x)}{\partial x_i} e^{-\beta \bar f(x)}
$$
so that
$$
\psi(x)\omega(x) = \sum_{i=1}^n (-1)^{i+1} \alpha_i(x) \bigwedge_{j\ne i} dx_j.
$$

Let $J(y)$ be the Jacobian matrix of $\phi^{-1}(y)$ and let $M_{ij}(y)$ be the associated minors. Then the definition of the pullback followed by Proposition 14.11 of \cite{lee2013introduction} shows that:
\begin{align*}
  (\phi^\star(\psi\omega))(y)
  &= \sum_{i=1}^n(-1)^{i+1}\alpha_i(\phi^{-1}(y)) \bigwedge_{j\ne i} \left(
    \sum_{k=1}^n J_{jk}(y) dy_k
    \right) \\
  &=\sum_{i=1}^n (-1)^{i+1} \alpha_i(\phi^{-1}(y)) \sum_{k=1}^n M_{ik}(y) \bigwedge_{\ell \ne k} dy_\ell.
\end{align*}

As in the proof of Stokes theorem from \cite{lee2013introduction}, all of the terms of the pullback that include $dy_n$ integrate to zero. This is because the $y_n$ is fixed at $0$ on the boundary, so the integrals over $y_n$ must be zero. 
Thus, the integral simplifies to:
\begin{equation}
  \label{eq:droppedIntegral}
  \int_{\partial \bbH^n}((\phi^{-1})^\star(\psi\omega))(y) = \int_{\partial \bbH^n} \sum_{i=1}^n(-1)^{i+1} \alpha_i(\psi^{-1}(y)) M_{in}(y) \bigwedge_{j\ne n} dy_j
\end{equation}

To show that the right side is zero, it suffices to show that the integrand on the right is zero. The inverse function theorem, followed by Cramer's rule shows that 
\begin{equation*}
  \left.\frac{\partial \phi_n}{\partial x_i} \right\vert_{x=\phi^{-1}(y)}=
  (J(y)^{-1})_{ni} = \frac{1}{\det(J(y))}(-1)^{i+n} M_{in}(y). 
\end{equation*}

Letting $x=\phi^{-1}(y)$, it follows that the integrand on the right of (\ref{eq:droppedIntegral}) is given by
\begin{equation*}
  \sum_{i=1}^n(-1)^{i+1} \alpha_i(\phi^{-1}(y)) M_{in}(y) = \det(J(y))(-1)^{1-n}\psi(x)e^{-\beta \bar f(x)} \nabla h(x)^\top \nabla \phi_n(x).
\end{equation*}
Note here that if $y\in\partial\bbH^n$, then $x\in \partial \cK$. So, (\ref{eq:tangent}) implies that the integrand is zero if $-\nabla \phi_n(x)\in N_{\cK}(x)$. This follows because for all $z\in \cK$ and all $t>0$ sufficiently small, we have that $x+t(z-x)\in \cK$ and
$$
0\le \phi_n(x+t(z-x)) = \phi_n(x)+t\nabla \phi_n(x)^\top (z-x) +o(t) = t\nabla \phi_n(x)^\top (z-x)+o(t).
$$
Thus $\nabla \phi_n(x)^\top x \le \nabla \phi_n(x)^\top z$ and so $-\nabla \phi_n(x)\in N_{\cK}(x)$. Thus, (\ref{eq:instantaneousExplicit}) has been proved and so the lemma has been proved for $\cK$ with smooth boundaries.

Now we cover the general case. Let $b$ be a self-concordant barrier function for $\cK$ such that for any sequence $x_i\in \mathrm{int}$ with $x_i\to \partial \cK$, we have $b(x_i)\to\infty$. By Theorem~2.5.1 of \cite{nesterov1994interior}, such a barrier function exists. Let $\cK_i = \{x| b(x)\le i\}$. For all sufficiently large $i$, we have that $\cK_i$ is non-empty. Whenever $\cK_i$ is nonempty, it has a smooth boundary. Furthermore, if $S$ is a compact subset of $\mathrm{int}(\cK)$, then $S\subset \cK_i$ for all sufficiently large $i$.  Let $\bx^{M,i}$ be the Skorokhod solutions corresponding to the sets $\cK_i$.

Let $Z_i = \int_{\cK_i} e^{-\beta \bar f(x)}dx$. Then $Z_i\uparrow Z=:\int_{\cK} e^{-\beta \bar f(x)}dx$ by monotone convergence. Let $\bx^{M,i}$ be the solution from (\ref{eq:averagedLangevin}) in which the Skorokhod problem is solved over $\cK_i$ in place of $\cK$. 
Let $P_t^i$ be the diffusion operator corresponding to $\bx_t^{M,i}$. Then,
for each non-empty $\cK_i$, the corresponding version of (\ref{eq:stationaryTest}) can be written as
\begin{equation}
\label{eq:approxStationarityTest}
\frac{1}{Z_i} \int_{\cK_i} g(x) e^{-\beta \bar f(x)} dx = \frac{1}{Z_i}\int_{\cK_i}(P_t^ig)(x) e^{-\beta \bar f(x)}dx 
\end{equation}
Using the fact that $g$ is bounded on $\cK$, dominated convergence implies that 
$$
\lim_{i\to\infty}\frac{1}{Z_i} \int_{\cK_i} g(x) e^{-\beta \bar f(x)} dx = \frac{1}{Z}\int_{\cK} g(x)e^{-\beta \bar f(x)}dx = \int_{\cK} g(x) d\pi_{\beta \bar f}(x).
$$
The proof will be completed if we can show that the right side of (\ref{eq:approxStationarityTest}) converges to the right side of (\ref{eq:stationaryTest}).

Note that the right side of (\ref{eq:approxStationarityTest}) can be expressed as
$$
\frac{1}{Z_i}\int_{\cK_i}(P_t^ig)(x) e^{-\beta \bar f(x)}dx =\frac{1}{Z_i}\int_{\cK_i}\bbE[g(\bx_t^{M,i})|\bx_0^{M,i}=x]e^{-\beta \bar f(x)}dx
$$
Now, Lemma~\ref{lem:sdeApprox} from Appendix~\ref{appsec:skorohod} shows that for almost all realizations of the Brownian motion, $\bw$, the Skorokhod solution converges pointwise $\lim_{i\to \infty}\bx_t^{M,i}=\bx_t^M$ for all $t\ge 0$. (In fact it converges uniformly on compacts.) As a result the integrand on the right converges pointwise almost surely to the integrand on the right side of (\ref{eq:stationaryTest}). Thus, the integrals on the right of (\ref{eq:approxStationarityTest}) converge to the integral on the right of (\ref{eq:stationaryTest}) via dominated convergence.
\end{proof}

\section{An Elementary Result on Stieltjes Integration}

\label{app:integration}

The following basic result is used a few times to examine bounded variation  functions, $x(t)$ whose differentials $dx(t)$ are only known when $x(t)\ne 0$. 

\begin{lemma}
  \label{lem:variationAtZero}
  Let $x(t)$ be a continuous non-negative function with bounded variation. Then
  \begin{equation}
    x(t) -x(0)= \lim_{\epsilon \downarrow 0} \int_0^t \indic(x(s) \ge \epsilon) dx(s)
  \end{equation}
\end{lemma}

\begin{proof}
  Fix any $\epsilon >0$. Then we have the Stieltjes integral representation:
  \begin{align*}
    x(t)-x(0) &= \int_0^t dx(s) \\
    &= \int_0^t \indic(x(s)\ge \epsilon) dx(s) + \int_0^t \indic(x(s)<\epsilon)dx(s).
  \end{align*}
  We will show that the second integral on the right goes to $0$ as $\epsilon \to 0$. This would imply the desired result by re-arranging and taking limits. 

  Fix any $\delta >0$. Then since $x(t)$ has bounded variation, there are numbers $0=s_0 < s_1 < \cdots < s_{N} = t$ such that
  \begin{equation*}
    \left|\int_0^t \indic(x(s)<\epsilon)dx(s) - \sum_{i=0}^{N-1} \indic(x(\bar{s}_i) <\epsilon) (x(s_{i+1})-x(s_i)) \right| \le \delta,
  \end{equation*}
  where $\bar{s}_i=\frac{1}{2}(s_{i+1}+s_i)$.

  Continuity of $x(t)$ implies that the $s_i$ can be chosen such that if $x(\bar{s}_i) < \epsilon$ then
  $x(s_{i})\le \epsilon$ and $x(s_{i+1})\le \epsilon$. 

  If $x(\bar{s}_{\hat i})<\epsilon$ for some $\hat i$, then let $\cI=\{j,j+1,\ldots,k\}$ be the largest sequence of integers such that $0\le j \le \hat i\le k\le N-1$ and $x(\bar{s}_i)<\epsilon$ for $i=j,\ldots,k$. Then
  \begin{equation}
    \label{eq:groupedTelescope}
    \sum_{i=j}^{k} \indic(x(\bar{s}_i) <\epsilon) (x(s_{i+1})-x(s_i)) = x(s_{k+1})-x(s_j)
  \end{equation}
  Maximality of the interval and our choice of $s_i$ imply that either $j=0$ or $x(s_{j}) = \epsilon$ and either $k+1=N$ or $x(s_{k+1})=\epsilon$.

  Note that in all cases $|x(s_{k+1})-x(s_j)| \le \epsilon$. If $x(s_{j})=x(s_{k+1})=\epsilon$ then the sum from \eqref{eq:groupedTelescope} is zero. So the sum can only be non-zero if $j=0$ or $k+1=N$ (or both).

  Since every term such that $x(\bar{s}_i)<\epsilon$ can be included in one of the intervals constructed above, and a most two of them can give rise to a non-zero sum, we see that the Riemann sum is bounded as:
  \begin{equation*}
\left| \sum_{i=0}^{N-1} \indic(x(\bar{s}_i) <\epsilon) (x(s_{i+1})-x(s_i)) \right| \le 2\epsilon
  \end{equation*}

 Using the fact that $\delta$ is arbtrary and using the triangle inequality shows that 
  $$
  \left|
 \int_0^t \indic(x(s)<\epsilon)dx(s)
  \right| \le 2\epsilon
  $$
\end{proof}

\end{document}